\documentclass{article}

\usepackage{microtype}
\usepackage{graphicx}
\usepackage{subfigure}
\usepackage{booktabs} % for professional tables
\usepackage{caption}
\usepackage{amsmath} 
\usepackage{amsthm}
\newtheorem{theorem}{Theorem}
\newtheorem{proposition}{Proposition}
\usepackage{mathtools}
\usepackage{amssymb}
\usepackage{adjustbox}
\usepackage{hyperref}

\newcommand{\beginsupplement}{%
        \setcounter{table}{0}
        \renewcommand{\thetable}{S\arabic{table}}%
        \setcounter{algorithm}{0}
        \renewcommand{\thealgorithm}{S\arabic{algorithm}}%
        \setcounter{equation}{0}
        \renewcommand{\theequation}{S\arabic{equation}}%
        \setcounter{figure}{0}
        \renewcommand{\thefigure}{S\arabic{figure}}%
        \setcounter{section}{0}
        \renewcommand{\thesection}{S\arabic{section}}%
     }

\usepackage[accepted]{icml2021}

\icmltitlerunning{On-Off Center-Surround Receptive Fields for Image Classification}

\begin{document}

\twocolumn[
\icmltitle{On-Off Center-Surround Receptive Fields\\ for Accurate and Robust Image Classification}



\icmlsetsymbol{equal}{*}

\begin{icmlauthorlist}
\icmlauthor{Zahra Babaiee}{to}
\icmlauthor{Ramin Hasani}{goo}
\icmlauthor{Mathias Lechner}{ed}
\icmlauthor{Daniela Rus}{goo}
\icmlauthor{Radu Grosu}{to}
\end{icmlauthorlist}

\icmlaffiliation{to}{CPS, TU Wien}
\icmlaffiliation{goo}{CSAIL, MIT}
\icmlaffiliation{ed}{IST Austria}

\icmlcorrespondingauthor{Zahra Babaiee}{zahra.babaiee@tuwien.ac.at}
%\icmlcorrespondingauthor{Eee Pppp}{ep@eden.co.uk}

\icmlkeywords{Machine Learning, ICML}

\vskip 0.3in
]



\printAffiliationsAndNotice{} % otherwise use the standard text.

\begin{abstract}
Robustness to variations in lighting conditions is a key objective for any deep vision system. %and a main enabler for many advanced robotic applications, such as autonomous driving. 
To this end, our paper extends the receptive field of convolutional neural networks with two residual components, ubiquitous in the visual processing system of vertebrates: On-center and off-center pathways, with excitatory center and inhibitory surround; OOCS for short. The on-center pathway is excited by the presence of a light stimulus in its center, but not in its surround, whereas the off-center one is excited by the absence of a light stimulus in its center, but not in its surround. We design OOCS pathways via a difference of Gaussians, with their variance computed analytically from the size of the receptive fields. OOCS pathways complement each other in their response to light stimuli, ensuring this way a strong edge-detection capability, and as a result an accurate and robust inference under challenging lighting conditions. We provide extensive empirical evidence showing that networks supplied with the OOCS edge representation gain accuracy and illumination-robustness compared to standard deep models. 

% This paper introduces OOCS-CNNs, a new type of convolutional neural networks whose exquisite edge-detection capabilities improves CNN classification accuracy and robustness in the presence of challenging lighting conditions. OOCS extends the CNN receptive-fields with two components ubiquitous in the visual processing system of vertebrates: On-center and Off-center parallel pathways, with excitatory center and Inhibitory surround; OOCS for short. The On-center pathway is excited by the presence of a light stimulus in its center, but not in its surround, whereas the Off-center pathway is excited by the absence of a light stimulus in its center, but not in its surround. Hence, the two pathways complement each other in their response to light stimuli.
% %and they together equip CNNs with a remarkable accuracy and robustness, in object recognition tasks under challenging illumination conditions. 
% We demonstrate the superior behavior of OOCS-CNNs on Imagenet, Norb, and Inverted-MNIST datasets, in experiments with noisy inputs and different lighting conditions. Our proposed OOCS-CNNs proved to outperform state of the art regularization techniques in image recognition tasks.
%\todoforramin{Remove the word 'motifs'. remove the a)b), just have a sentence with On/Off-center pathways and excitatory/inhibitory surrounding. Strong edge-detection in challening lighting condition as goal of this architecture should be emphasized much earlier in the abstract. Replace MNIST by Fashion-MNIST. Edit: I suggest to call it ''inverted-MNIST'' (see experiment section)}
\end{abstract}

\vspace*{-3.5ex}
\section{Introduction}
\label{sec:Introduction}

The great success of convolutional neural networks (CNNs) \cite{fukushima_2003} is rooted in receptive fields, a main architectural motif of visual processing in living organisms~\cite{mack2013principles}. Originating in the retina, a receptive field defines the region of visual space within which visual stimuli affect the firing of a single ganglial neuron~\cite{hartline_1940,mack2013principles}. This motif is preserved by neurons of the visual cortex, too~\cite{hubelWiesel1968}.

However, receptive fields are just one of the motifs employed by visual processing in the retina. 
%The receptive field of a retinal ganglion cell is the area of the retina, which in response to the stimuli, can modify the spike discharge of that ganglion cell.~\cite{hartline_1940}. There are distinct subareas that function differently in the receptive fields of the retinal ganglion cells. 
Another important motif is the center-surround (CS) motif, which divides the receptive field of a ganglial neuron into a circular excitatory region (the center), and a concentric inhibitory region (the surround)~\cite{doi:10.1152/jn.1953.16.1.37,mack2013principles}. 
\begin{figure}[t]
\centering
\includegraphics[width=0.45\textwidth]{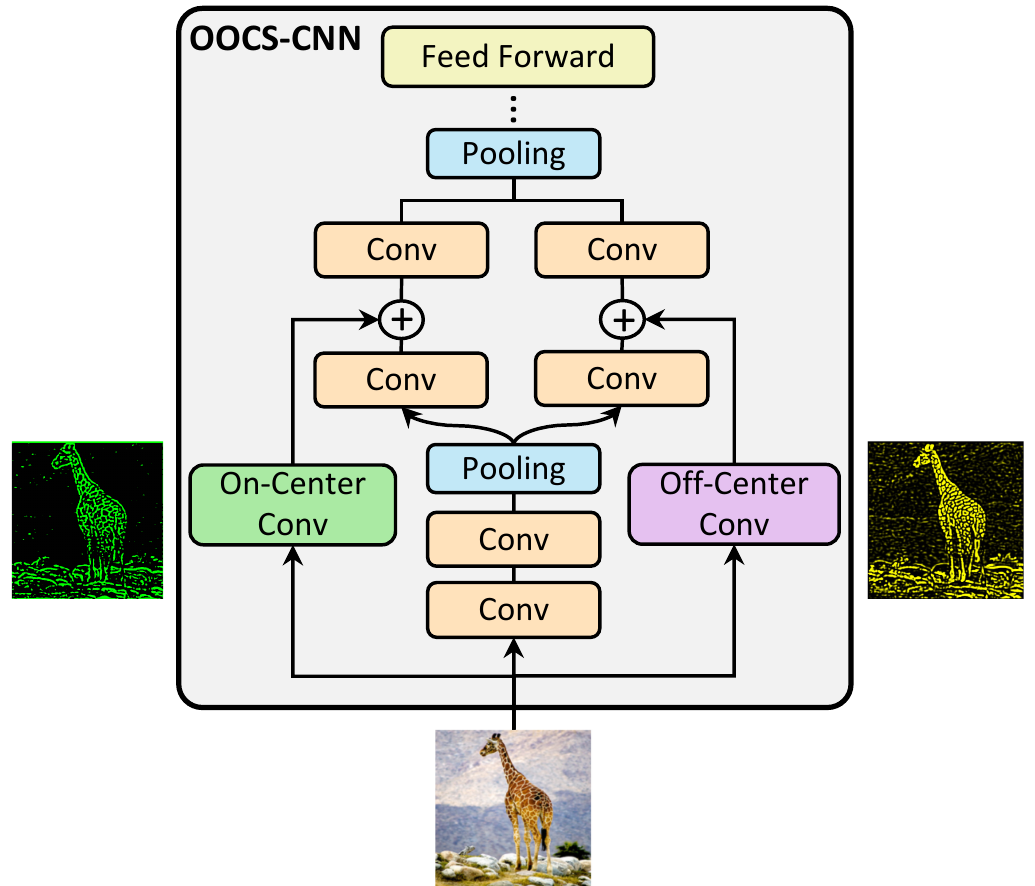}
\vspace*{-1.5ex}
\caption{On-off center-surround (OOCS) inductive biases significantly enhance the performance and robustness of a vision network in different lighting conditions.}
\label{fig:oocsArchitecture}
\vspace*{-2ex}
\end{figure}
%
%These subareas divide the receptive field to a circular central region(center) and a concentric ring(surround).~\cite{doi:10.1152/jn.1953.16.1.37} 
Finally, a third important motif classifies the center-surround fields into either on-center, when the neurons fire in response to the presence of a light stimulus at the center of their receptive field,  and into off-center (OO), when the neurons fire in response to the absence of a light stimulus at the center of their receptive field~\cite{enroth-cugell_pinto_1972}. The center-surround motif is also reported to occur in the receptive fields of the primary level visual cortex, causing the so called surround-modulation effect~\cite{doi:10.1152/jn.1965.28.2.229,doi:10.1152/jn.1992.67.4.961,doi:10.1146/annurev.ne.08.030185.002203}. 
%The retinal ganglion cells fall into one of two categories based on their receptive fields: on-center and off-center.
%
\begin{figure}[t]
\centering
\includegraphics[width=0.4\textwidth]{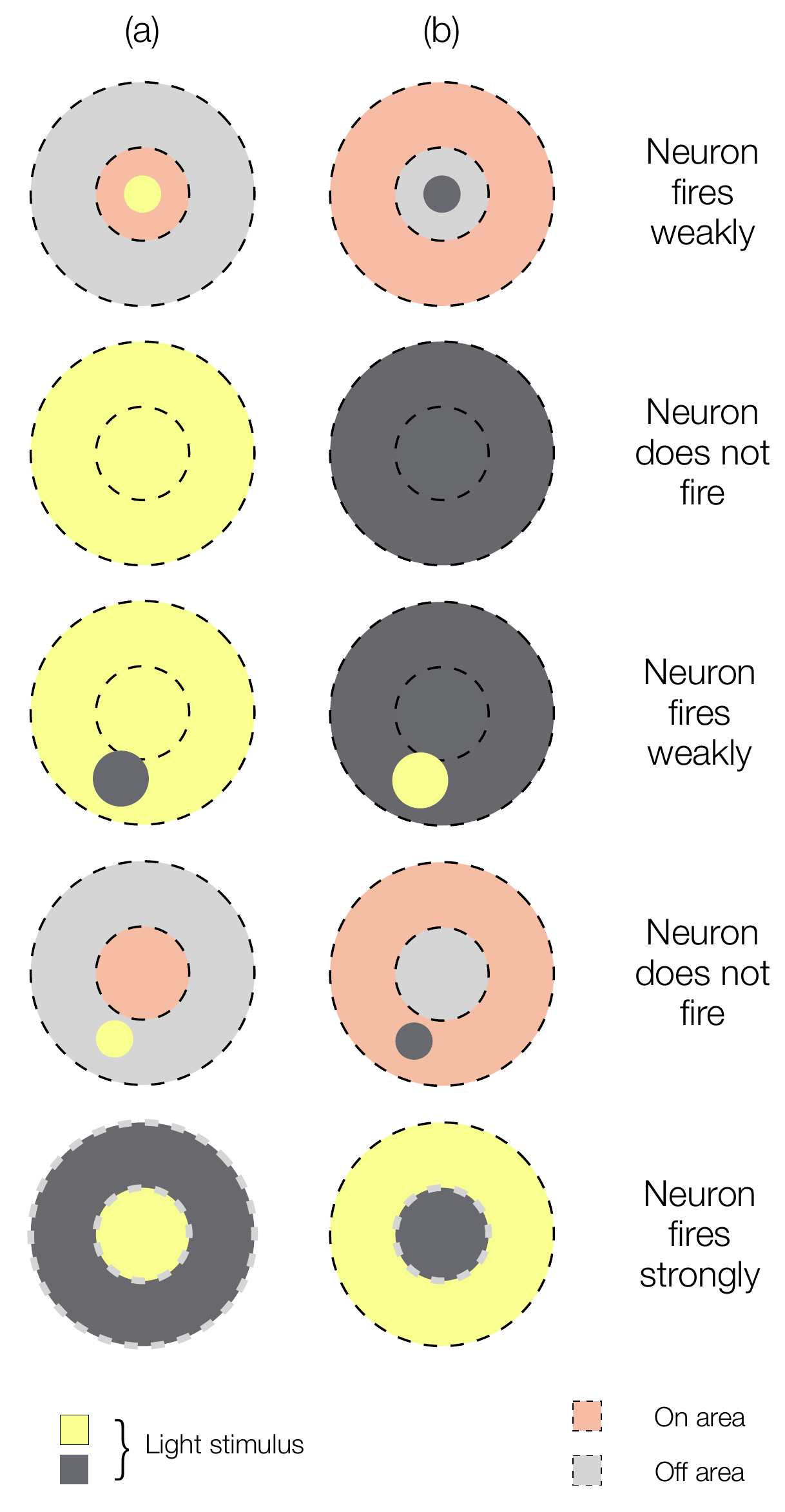}
\vspace*{-1.5ex}\caption{\cite{mack2013principles}~ The response of a) On-center neurons and b) Off-center neurons and their associated receptive fields, respectively, to five different lighting conditions.}
\label{fig:receptivefields}
\vspace*{-2ex}
\end{figure}

As shown in Figures~\ref{fig:oocsArchitecture}-\ref{fig:images}, the on-off center-surround pathways (OOCS) introduce edge detection inductive biases in a given model. If a light stimulus is turned on at the center of an on-receptive field, but not at its surround, then the on-ganglial neuron fires vigorously. If the light also touches the surround, then the neuron yields a weak response, and ceases to fire in case of a uniform or a surround-only stimulation. This is because of the mutual inhibition of the center and its surround. Complementary, when the light stimulus at the center of an off-receptive field is turned off, but not at its surround, then the off-ganglial neuron fires vigorously. The neuron ceases to fire when the light is turned off at the surround, too~\cite{enroth-cugell_pinto_1972,mack2013principles}.
Various studies of retinal circuity have shown that the on-center and off-center ganglial neurons, respectively, are at the origin of two parallel pathways~\cite{doi:10.1113/jphysiol.2005.088047, Zaghloul2645,SHAPLEY1986229}. They are physiologically and anatomically distinct, and their receptive fields cover the retinal area completely~\cite{Dacey2004-DACOO,mack2013principles}.

%The on-center cells fire when there is a light stimulus at the center of their receptive field. When the light moves to the surrounding area, the cell gets inhibited. The off-center cells act oppositely. When the light stimulus at the center of their receptive field turns off, they fire. In other words, when there is a dark spot in the central region, they get excited. Similarly to the on-center cells, the surrounding area has the opposite sign.~\cite{ enroth-cugell_pinto_1972,mack2013principles}.
%
%A uniform surface stimulates both the center and surround of the receptive fields. Because of the mutual inhibition of center and surround, this yields a weak response in the ganglion neurons. Conversely, these neurons strongly fire when the center gets a different light than the surround. Since the receptive fields are overlapping, this organization lets the neurons to be more sensitive to dark-light borders and contours and respond strongly to the edges. Therefore they transmit useful information about the contrast. They also help the retina to stay stable through ever-changing lighting conditions each day.~\cite{mack2013principles}

Motivated by OOCS in living organisms, we developed a procedure extending any CNN to an OOCS-CNN with same number of parameters, as shown in Figure~\ref{fig:oocsArchitecture}. This first adds the complementary CS convolutional kernels to the input processing. The kernels are fixed and precomputed. Their results are added after one of the original convolutional layers (which intuitively convolve with different kernels, possibly of different sizes). In  Figure~\ref{fig:oocsArchitecture} we add them after the third layer, which worked best on Imagenet. We also split the original layers in two, and concatenate the results thereafter, at the place of addition. 
%and investigated their accuracy and robustness. OOCS do not increase the number of parameters in a CNN, and can be easily applied to different network types.  
%
Our experiments show that 
%similarly to their biological counterpart, 
OOCS improves both the performance of CNNs in image recognition tasks, and the robustness of CNNs to challenging illumination conditions. In our 
%object recognition 
experiments with unseen lighting and noisy test sets OOCS-CNNs also outperformed other regularization methods.
%The response of on-center and off-center neurons is complementary. When one of them fires rapidly, the other remains quiet. 

\noindent{\bf Summary of Contributions:}
%We introduce OOCS-CNNs, give their theoretical foundation, and thoroughly evaluate their accuracy and robustness to various lighting conditions. In more detail:
%
\vspace*{-2mm}
\begin{itemize}
\itemsep-0.2em
%\vspace*{-3mm}
%\item We introduce an advanced center-surround modulation in CNNs that is superior to state-of-the art approaches.
\item We introduce OOCS, an inductive edge-detection bias for enhancing the performance and robustness to the variation of lighting conditions of vision networks.
%\vspace*{-6mm}
%\item We prove that On and Off pathways capture complementary features, lost in the On or Off pathways standalone.
\item We prove that on and off residual pathways capture complementary features that improve the generalization error and robustness to distribution shifts.
%\vspace*{-6mm}
\item We show that OOCS can be applied to any CNN architecture without increasing their number of parameters.
%\vspace*{-6mm}
%\item We experimentally show that OOCS-CNNs are more accurate and more robust than state-of-the-art CNNs.
\item We conduct an extensive set of experiments showing the superior performance and robustness of OOCS networks compared to standard deep models.
\end{itemize}

\section{Related Work}
\textbf{Receptive field in CNNs.} %CNNs' successive layers extract in a hierarchical tree-like fashion visual features, that have an increasingly global range in the original image. 
In each convolution layer, a small-sized kernel shifts over the input image, convolving each patch beneath it with the kernel matrix. These kernels were inspired by, and function like the receptive fields \cite{luo2016understanding,li2019scale}: they change the activity of the neurons connected to that patch in the next layer~\cite{hubelWiesel1968,li2019selective}. Fukushima's Neocognitron~\cite{fukushima_1980} is arguably the first CNN model to have imported the concept of receptive fields from neural science~\cite{hubelWiesel1968}, and inspired a large body of CNN variants~\cite{fukushima_2003, doi:10.1162/neco.1989.1.4.541, Lecun98gradient-basedlearning,inproceedings,wang2019kervolutional,ding2019acnet,hu2018squeeze}. 

\textbf{Bio-inspired Models.}
A large body of works tried to bringing insights from neuroscience to computer vision systems \cite{kim2016convolutional,zoumpourlis2017non,laskar2018correspondence,lechner2020neural,Hasani2021liquid}. Fukushima enhanced the Neocognitron model with several priors \cite{jacobsen2016structured}, such as contrast-extracting preprocessing layer, inspired by the On-Off ganglial neurons, and inhibitory surround connections, such as the surround-modulation in the visual cortex~\cite{fukushima_2003}.
%While the Neocognitron had a preprocessing contrast-extractor layer, partially inspired by the on-fff cells~\cite{fukushima_2003},
More recent bio-inspired work proposed to replace the feed-forward architecture of CNNs with a recurrent architecture, by adding local-range and long-range feedback connections~\cite{NIPS2018_7775}, or by adding intrinsic horizontal connections~\cite{NIPS2018_7300}.
%Some do not touch the feed-forward structure of CNNs. 

A center-surround architecture is also proposed in~\cite{NIPS2019_9719},
%While surround modulation was already introduced in the Neocognitron in form of inhibitory surround connections~\cite{fukushima_2003}, this work proposes a connectivity structure 
in the form of a convolutional layer with a fixed kernel (a difference of Gaussians DoG, as shown in Figures~\ref{fig:3dplots}(a) and~\ref{fig:kernels}(a-b)), to the activation map of the first layer of the network. In a fixed kernel, positive weights are introduced for close neighbor neurons, and negative weights for those that are farther apart. This approach improved the performance of CNNs on various image-classification tasks.
%~\cite{NIPS2019_9719}.

%In contrast to~\cite{fukushima_2003} and 
Similar to~\cite{NIPS2019_9719}, we also use a DoG kernel with positive weights for the center neurons and negative weights for the surround neurons, as shown in Figures~\ref{fig:3dplots}(a) and~\ref{fig:kernels}(c). However, while~\cite{NIPS2019_9719} uses a DoG kernel similar to~\cite{rodieck_1965}, our work uses the DoG kernel in~\cite{Petkov2005ModificationsOC, kruizinga_petkov_2000}. The disadvantage of~\cite{rodieck_1965} is that the variances of the Gaussians are unknown. For example, in~\cite{NIPS2019_9719} they are fixed to 1.2 for inhibitory synapses and to 1.4 for excitatory ones.  In addition, the DoG in~\cite{rodieck_1965} may result in very small numbers, that have to be normalized as in~\cite{NIPS2019_9719}, for achieving meaningful results. Using the DoG
in~\cite{Petkov2005ModificationsOC, kruizinga_petkov_2000} we do not need to perform a hyperparameter search to find the optimum value for the variances. By just knowing the size of the receptive fields, we can analytically compute the corresponding, large enough, variances.
%In contrast, by requiring that the positive part of the DoG sums up to one, and the negative part to minus one, the DoG in~\cite{Petkov2005ModificationsOC, kruizinga_petkov_2000} results in meaningful, large enough values.
Finally, in contrast to~\cite{NIPS2019_9719}, we use both on- and off-pathways \cite{kim2016convolutional}, as shown in Figure~\ref{fig:3dplots}, and in Figure~\ref{fig:kernels}(c) and its complement. These pathways capture complementary features which are lost in the use of either on- or off-pathways alone.

\section{Main Results}
% Finally, in contrast to~\cite{fukushima_2003,NIPS2019_9719}, and to the best of our knowledge, this paper adds for the first time On-Off parallel pathways to CNNs, and thus the full OOCS paradigm present in the visual system of living organisms. This considerably increases both the accuracy and the robustness of OOCS-CNNs when compared to the Center-Surround-only approach in~\cite{NIPS2019_9719}. The performance of the latter in noisy conditions rapidly degrades, whereas the one of OOCS-CNNs remains very stable.

%The rest of the paper is organized as follows. In next section we introduce the Center-Surround (CS) and the On-Off (OO) architectural motifs, give their mathematical formulation, and discuss how to extend CNNs with these motifs. We then provide a thorough evaluation of the OOCS-CNNs accuracy and robustness on the Imagenet~\cite{DBLP:conf/cvpr/DengDSLL009} and Norb~\cite{DBLP:conf/cvpr/LeCunHB04} datasets, in the following section. 
%In Section~\ref{sec:Conclusions} we draw our conclusions and discuss future work. 
%Finally, in the Broader Impact section, we discuss the impact of this work on other areas.
\label{sec:Model}
In this section, we discuss our main findings. We first introduce the structure of the OOCS blocks. We then lay the theoretical grounds for their effectiveness in robustifying deep models to the variation of lighting conditions. 

\begin{figure}[t]
   \includegraphics[width=0.45\textwidth]{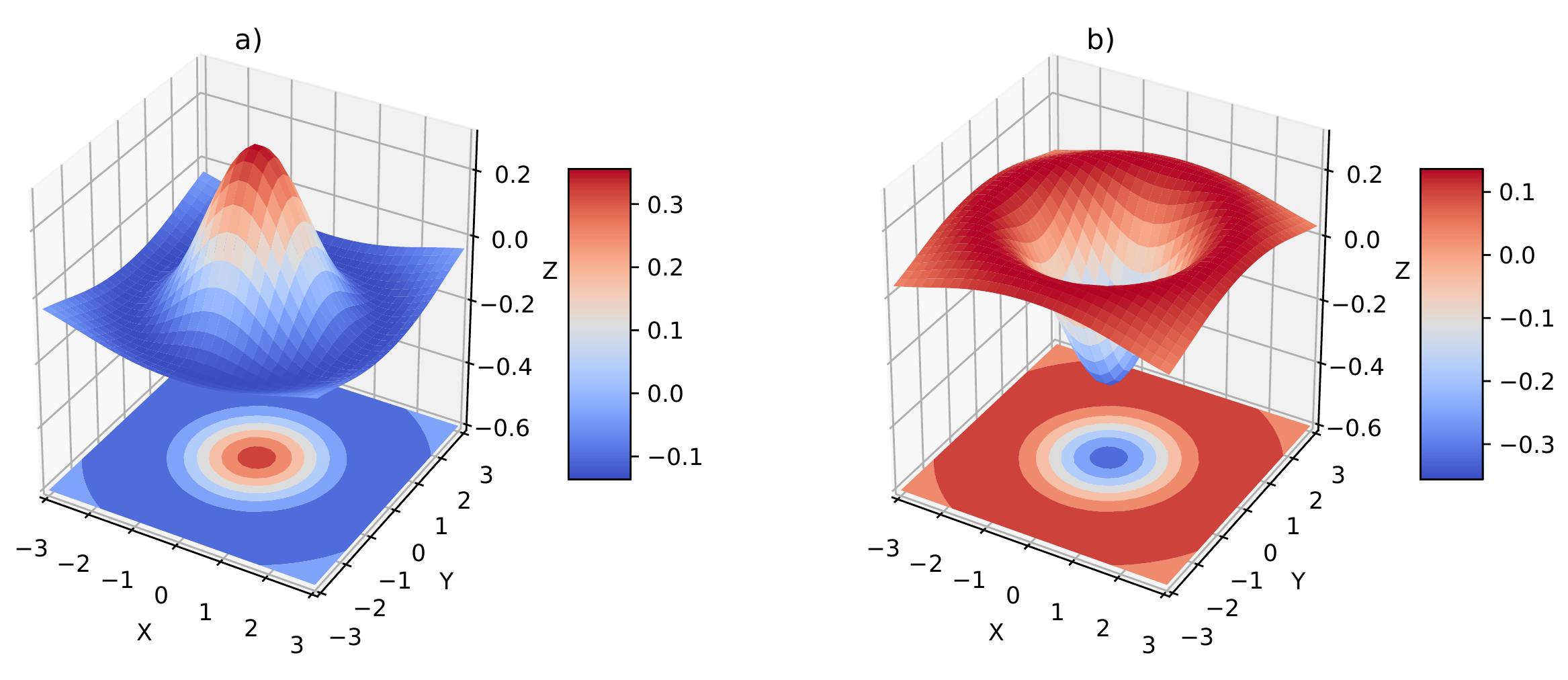}
   \vspace*{-2ex}
   \caption{Our difference of Guassians function for a) On and b) Off kernels of size 5 and center to surround ratio of $2/3$.}
    \label{fig:3dplots}
   \vspace*{-2ex}
\end{figure}

\textbf{Center Surround Kernels.} Center and surround weights can be computed by a difference of two Gaussian functions (DoG). This difference can be written in Cartesian coordinates (CC), as follows, where the CC origin is taken as the center of the receptive field \cite{rodieck_1965}:
\begin{equation}
\centering
\label{eq:rodiek}
DoG(x,y) = K_1\, e^{-\frac{x^2+y^2}{\sigma_1}} - K_2\,
e^{-\frac{x^2+y^2}{\sigma_2}}
\end{equation}
where $K_1\,{>}\,K_2$ and $\sigma_2\,{>}\,\sigma_1$~\cite{Blackburn1993ASC}. 

\textbf{Why the kernel presented by (\ref{eq:rodiek}) is not optimal?}
The main disadvantage of this model is that the variances of the Gaussians have to be determined through a hyper-parameter search. The values of the weights calculated are typically too small, and have to be normalized like in~\cite{NIPS2019_9719}, to obtain meaningful results, as in Figure~\ref{fig:kernels}(b). 
\begin{figure}[t]
    \centering
   \includegraphics[width=0.5\textwidth]{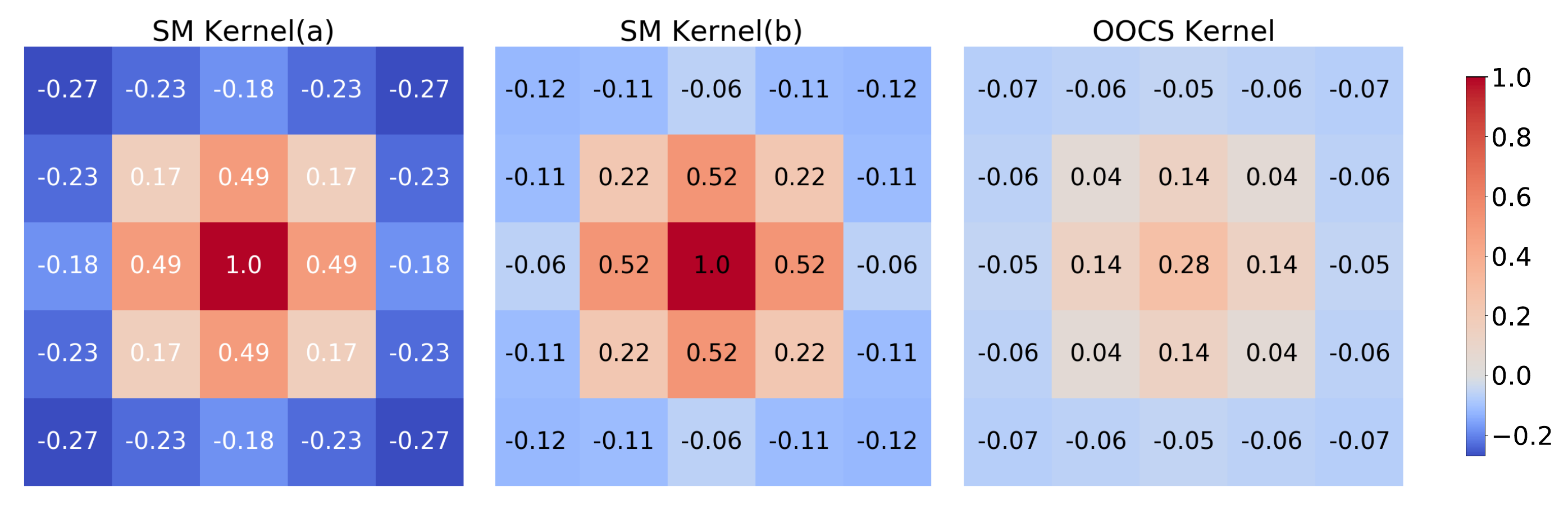}
   \vspace*{-4ex}
   \caption{From left to right: Kernel taken from~\cite{NIPS2019_9719} used in their experiments; Actual kernel calculated by their proposed method; Our on-center kernel.}
    \label{fig:kernels}
   \vspace*{-2ex}
\end{figure}
Furthermore, with $K_i = 1/2\pi\sigma_i^2$ which is used in \cite{NIPS2019_9719}, the two Gaussian functions take almost equal values for close $\sigma_1$ and $\sigma_2$. This yields their differences taking very small values everywhere. Normalizing the values to the value of the center, results in loosing the center's weight in excitation or inhibition.

\textbf{Improved Center-Surround Kernels.} To overcome the shortcomings discussed above, we set out to design an improved kernel, as shown in Figures~\ref{fig:3dplots}(a-b) and \ref{fig:kernels}(c), based on the DoG model proposed in \cite{Petkov2005ModificationsOC,kruizinga_petkov_2000}, for defining the difference of Gaussians for the center and surround kernels. This model allows us to analytically compute the variances, from the size of the receptive fields \cite{Petkov2005ModificationsOC}:
\begin{equation}
\label{eq:petrov}
DoG_{\sigma,\gamma}(x,y) =  \frac{A_c}{\gamma^2}\, e^{-\frac{x^2+y^2}{2\gamma^2\sigma^2}} - A_s\,e^{-\frac{x^2+y^2}{2\sigma^2}}
\end{equation}
Here, $\gamma$ with $\gamma\,{<}\,1$, defines the ratio between the radius $r$ of the center and that of the surround. We determine the values of the coefficients $A_c$ and $A_s$ by requiring that the sum of all positive values in Equation~\eqref{eq:petrov} are equal to those of the negative values. We normalize this and let their sum be equal to 1 and -1, respectively:
\begin{figure*}[t]
    \centering
   \includegraphics[width=1\textwidth]{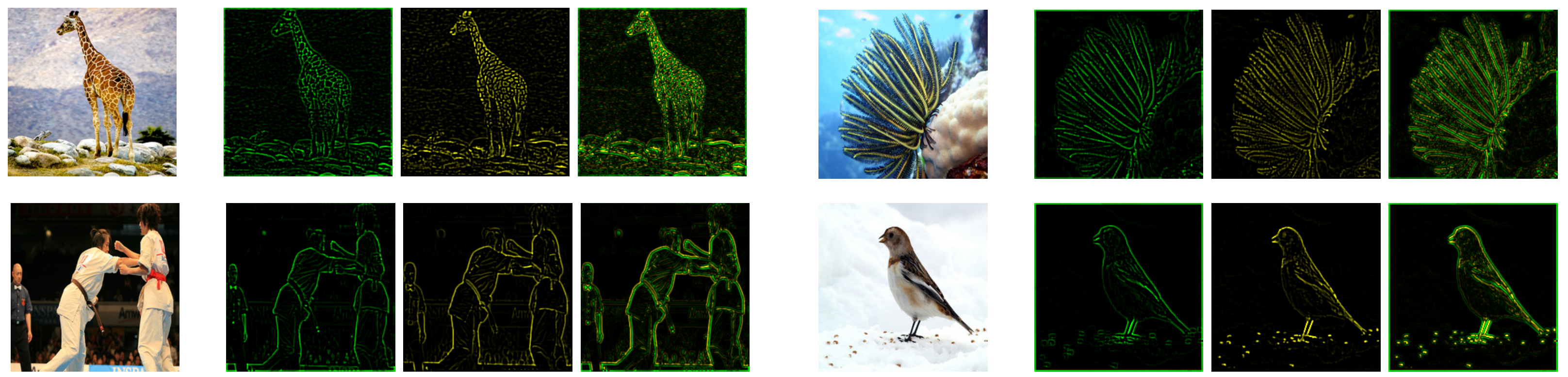}
   \caption{On-center and off-center kernel convolution on samples from Imagenet dataset. The 1st column shows the original image, the 2nd and 3rd columns show on-center and off-center filters respectively. The 4th column shows the filters added together. They extract complementary features within the original image. For example, the  bright spot top-left in the Judo image, is properly detected only by the on-center filter, whereas the dark grains on the bottom of the bird image, only by the off-center filter. Hence both filters are necessary.}
    \label{fig:images}
    %\vspace*{-2ex}
\end{figure*}
\begin{equation}
\iint [DoG_{\sigma,\gamma}(x,y)]^{+} dx dy = 1,
\end{equation}
\begin{equation}
\iint [DoG_{\sigma,\gamma}(x,y)]^{-} dx dy = -1
\end{equation}
By $[z]^{+}$ and $[z]^{-}$ we denote the positive and the negative half wave rectification functions, respectively:
\begin{equation}
[z]^{+} = max(0,z),
\quad [z]^{-} = min(0,z)
\end{equation}
\vspace*{-4ex}
\begin{proposition}[DoG Coefficients]
\label{theo:infinity} In the infinite continuous case, the coefficients $A_c$ and $A_s$ are equal. 
\end{proposition}
The proof is given in the supplementary materials.
By setting the $DoG_{\sigma,\gamma}(x,y) = 0$, and using Proposition~\ref{theo:infinity}, we can immediately calculate $\sigma$ as in Equation~\ref{eq:sigma} below, for arbitrary $r$ and $\gamma$. Note that in the finite discrete case, the values of $A_c$ and $A_s$ are still very similar:
\begin{equation}
\label{eq:sigma}
\sigma \approx \frac{r}{2\gamma}\sqrt{\frac{1-\gamma^2}{-\ln{\gamma}}}
\end{equation}
See Supplementary Materials for the complete calculations.

This model not only allows us to find the variance of the kernels analytically, but also it overcomes the normalization issues of (\ref{eq:rodiek}). The weights calculated by Eq.~\eqref{eq:petrov} do not have to be normalized and can be obtained by the kernel size and the ratio of the center to surround. Moreover, forcing the positive and negative weights to sum up to 1 and -1 results in a balance between the excitations and inhibitions. There is evidence for this balance in neuroscience (see Figure~\ref{fig:receptivefields}). 

In the DoG model of Equation~\eqref{eq:rodiek} however, the positive and negative weights do not necessarily sum up to zero, unless in very large kernels. 
We calculated the weights of the SM kernel in Figure~\ref{fig:kernels}(b)
using this equation, with the  parameters reported in~\cite{NIPS2019_9719}. However, this kernel is different from the one in Figure~\ref{fig:kernels}(a), which is actually used in their experiments. This kernel is altered in a way that the positive and negative weights sum up to zero.
%
%The model we use instead, does not  in~\cite{Petkov2005ModificationsOC}, makes the negative and positive values to sum up to 1, and thus in this model there is no need to perform normalization. Moreover, having only the size of the receptive field and the canter-surround ratio, we can easily compute the $\sigma$ and the $DoG$. This makes us needless to search for optimum values for $\sigma_1$ and $\sigma_2$.

\textbf{On and Off kernel matrices.}
We use Equation~\eqref{eq:petrov} to compute the weights in the On-center kernel matrix $DoG_{\rm On}$. For the Off-center kernel $DoG_{\rm Off}$, we use the same equation with inverted signs. For a given input  $\chi$, we calculate the On and Off responses by convolving $\chi$ with the computed kernels separately:
\\
\begin{equation}
\chi_{\rm On}[x,y] = (\chi * DoG_{\rm On})[x,y],
\end{equation}
\begin{equation}
\chi_{\rm Off}[x,y] = (\chi * DoG_{\rm Off})[x,y]
\end{equation}
It is important to note that both the on-convolution and the off-convolution covers the entirety of the input image. 

\textbf{Designing OOCS Pathways.}
As shown in Figure~\ref{fig:oocsArchitecture}, we add the computed responses to the activation maps of the first layers in the on- and off-pathways, respectively. To implement the on-off residual maps, we split the CNN layers between the on and off parallel pathways, with half of the number of the filters of the original layers, in the layers of each pathway. Thus, the number of the training parameters remains unchanged. At the end of the pathways, we concatenate the activation maps of the last on and off layers, and feed this to the rest of network. The system is structured and trained via Algorithm~\ref{alg:1}. The on-off pathways are residual connections to facilitate the training process while enhancing the network's robustness to illuminations. 
% \begin{algorithm}
% \label{alg:on-off}
% \SetAlgoLined
% \SetKwData{Left}{left}\SetKwData{This}{this}\SetKwData{Up}{up}
% \SetKwFunction{Union}{Union}\SetKwFunction{FindCompress}{FindCompress}
% \SetKwInOut{Inputs}{input}\SetKwInOut{Output}{output}
% \Inputs{Input $I$ for On-Off responses calculation, Output $x$ of the last network layer, DoG center radius $r$, Center-Surround ratio $\gamma$.}
% \Output{Input to the next network layer $x_{new}$}
% $x_{\emph{on}}$ = Conv$_{1,1}$($x$) + Conv$_{\emph{on}}$($I$, $r$, $\gamma$)\;
% $x_{\emph{off}}$ = Conv$_{1,2}$($x$) + Conv$_{\emph{off}}$($I$, $r$, $\gamma$)\;
% $x_{\emph{on}}$ = Conv$_{2,1}$($x_{on}$)\;
% $x_{\emph{off}}$ = Conv$_{2,2}$($x_{\emph{off}}$)\;
% $x_{\emph{new}}$ = Concatenate($x_{on}$, $x_{\emph{off}}$)\;
% return $x_{\emph{new}}$\;
%  \caption{On-Off pathways \todoforramin{This is not an algorithm}}
% \end{algorithm}

\begin{algorithm}[t]
\caption{Building and Training OOCS networks}
\label{alg:1}
\begin{algorithmic}
\STATE \textbf{Inputs:} mini batches $\mathcal{B}$, Network $N$ with layers $L=\{l_1,\dots, l_n\}$ parametrized by $\theta$, $l_{i,1/2} =$ half of layer $l_i$ DoG center radius $r$, Center-Surround ratio $\gamma$.
\STATE \textbf{Output:} OOCS Network 
\FOR{$j$ in Number of Training Steps}
\FOR{$b$ \textbf{in} $\mathcal{B}$}
\STATE $x_{\emph{on}} = l_{1,1/2}(b) + b * DoG_{on}(b, r, \gamma)$;
\STATE $x_{\emph{off}} = l_{1,1/2}(b) + b * DoG_{off}(b, r, \gamma)$;
%\STATE $x_{\emph{off}}$ = Conv$_{1,2}$($x$) + Conv$_{\emph{off}}$($I$, $r$, $\gamma$)\;
% \STATE $x_{\emph{on}}$ = Conv$_{2,1/2}$($x_{on}$)\;
% \STATE $x_{\emph{off}}$ = Conv$_{2,2/2}$($x_{\emph{off}}$)\;
\STATE $x_{\emph{new}}$ = Concatenate($l_{2,1/2}(x_{on})$, $l_{1,1/2}(x_{\emph{off}})$)\;
\STATE Construct the rest of $N$ on $x_{new}$ as input.
\STATE $L_{total} = \sum_{j=1}^{T} L(y_j, \hat{y}_j)$,~~~$\nabla L(\theta) = \frac{\partial L_{tot}}{\partial \theta}$
\STATE $\theta = \theta - \alpha \nabla L(\theta)$
\ENDFOR
\ENDFOR
\STATE \textbf{Return} $N_\theta$
\end{algorithmic}
\end{algorithm}

\begin{figure*}[t]
    \centering
   \includegraphics[width=1\textwidth]{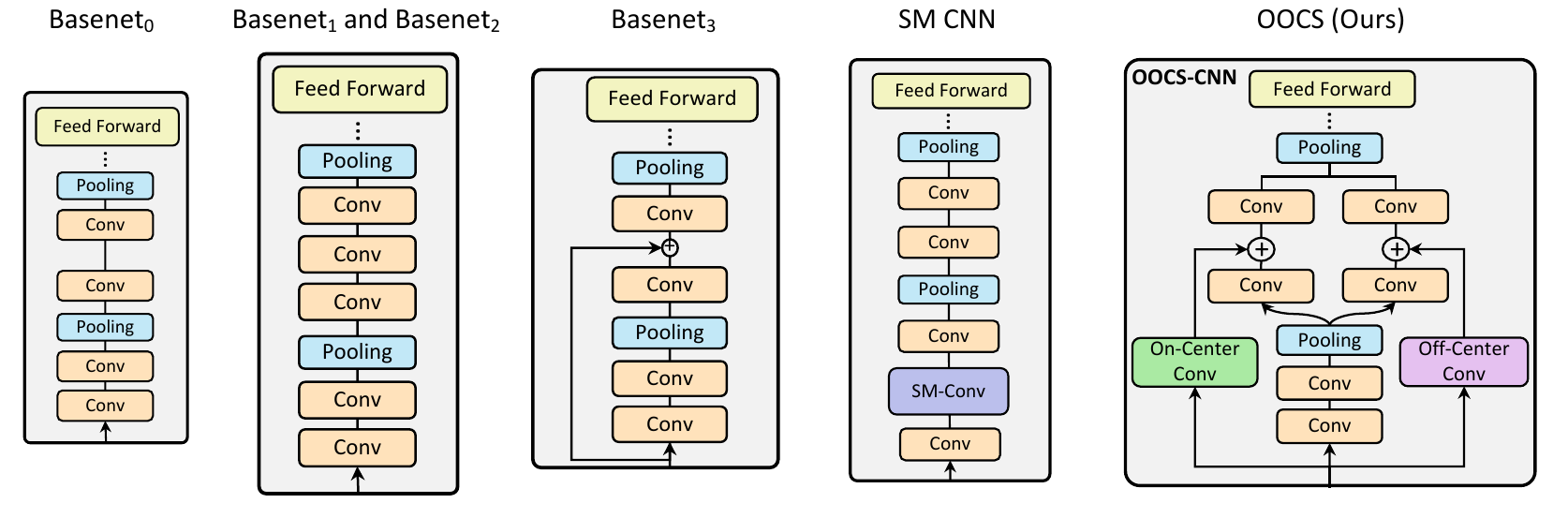}
\vspace*{-2ex}
   \caption{The network architecture for the Basenets, SM CNN and our OOCS-CNN.}
    \label{fig:architecture}
\vspace*{-2ex}
\end{figure*}

OOCS blocks can be used to extend any deep model, without the need to search for optimal hyperparameters for the on and the off kernels. Figure~\ref{fig:images} shows the saliency maps of on-convolution, off-convolution, and the combined convolution, for four samples taken from the Imagenet dataset~\cite{DBLP:conf/cvpr/DengDSLL009}. 
The OOCS kernels detect unique edge patterns, by extracting positive and negative contrasts. While they both detect edges, they do this in complementary ways. For a dark feature on a light background, the on-convolution detects the outer edges, and the off-convolution detects the inner ones. For a light feature on a dark background, roles are reversed. 
More importantly, the strength of the responses is different in the extracted edges, since the on-convolution and the off-convolution have a complementary response to light, as shown in Figure~\ref{fig:receptivefields}~(rows 1,3,5). This means that some of the features extracted by one of the convolutions are not present in the other, or they appear very weak. 
%In Figure~\ref{fig:images}, the light spot at the top-left corner of the judo image, is strongly detected by the On convolution and only weakly detected by the Off convolution. Complementary, the dark seeds on the snow in the bird image, are strongly detected by the Off convolution and only weakly detected by the On convolution.
%

\begin{theorem}[On-Off Complementarity]
\label{theo:main} 
The on- and off-path\-ways learn unique and complementary features.
%The features extracted by the On convolution are not identical to the features extracted by the Off convolution and vice versa.
\end{theorem}
The proof of Theorem~\ref{theo:main} is given in Supplementary Materials. It uses the definition and the associated properties (eg.,~Proposition~\ref{theo:infinity}) of the on- and off-convolutions. Theorem~\ref{theo:main} demonstrates that it is not enough to use either the on- or the off-convolution, alone. Both are required to maximize information flow, in accurate and robust image recognition.
%Details are in Supplementary Materials.

%Inspired by the visual-processing layers in vertebrates, the results of the parallel on- and off-pathways, are added to a CNN after its third layer. This addition mimics the role of the ganglion layer in the retina, which occurs after the horizontal, amacrine, and bipolar layers, respectively \cite{hartline_1940,mack2013principles}. In general, the optimal depth depends on the complexity of the data set. For Imagenet (complex), it was three, but for inverted MNIST (simple), zero depth (before fist convolution) gave best accuracy. 
%While in general, one may have to search for the optimal depth where to add these pathways (which may be different from three), the main objective of this work is to show the effectiveness of OOCS enhanced CNNs within standard settings.

\section{Experimental Evaluation}
\label{sec:Evaluation}
We perform two sets of experiments. In Section~\ref{subsec:Classification}, we evaluate OOCS on standard image classification benchmarks, and compare its performance to competitive baselines. We also perform a couple of ablation analyses on the baseline networks and OOCS to validate our observations. 

In Section~\ref{subsec:Robustness}, we perform a rigorous robustness analysis, by shifting the test distribution from the training set (Out of IID setting). In particular, we evaluate our models under various challenging lighting condition, as well as on the MNIST data set with black-white inverted images.
%
%The first, is a standard image classification task on the Imagenet dataset. The second, is an evaluation of the robustness of OOCS-CNNs, with test sets that have unfamiliar lighting conditions, or increasing noise. We compare OOCS-CNN results with state-of-the-art regularization methods. %The third experiment is an MNIST classification task with the colors inverted in the test set.

\begin{table}[t]
  \centering
  \vspace*{-2mm}
    \caption{Top-1 test accuracy and associated variance of the control models and our different OOCS models on Imagenet. n=6}
   \vspace*{4mm}
   \begin{adjustbox}{width=1\columnwidth}
  \begin{tabular}{ll}
    \toprule
       \textbf{Models}  & 
       \textbf{Accuracy} \\
    \midrule
      Basenet$_{0}$\hspace*{2.6ex}(standard deep CNN) & 40.8 $\pm$ 0.4 \\
      Basenet$_1$\hspace*{2.6ex}(extra convolution with ReLU) & 39.4 $\pm$ 0.6 \\
      Basenet$_2$\hspace*{2.6ex}(extra convolution)  & 41.0 $\pm$ 0.5\\
      Basenet$_3$\hspace*{2.6ex}(extra skip connection) & 42.1 $\pm$0.5  \\
      Basenet$_4$\hspace*{2.6ex}(5x5 kernel skip connection) & 42.0 $\pm$0.6  \\
      Basenet$_5$\hspace*{2.6ex}(OOCS without on/off kernels) & 41.3 $\pm$0.7  \\
    \midrule
      SM-CNN$_{0}$\hspace*{1ex}(kernel given in Figure~\ref{fig:kernels}a) &  38.1 $\pm$ 0.6    \\
      SM-CNN$_1$\hspace*{1ex}(kernel given in Figure~\ref{fig:kernels}b) &  37.0 $\pm$  0.4 \\
      SM-CNN$_2$\hspace*{1ex}(kernel given in Figure~\ref{fig:kernels}c) &  41.9 $\pm$  0.8 \\
    \midrule
      OOCS$_{0}$\hspace*{1ex}(on-kernel given in Figure~\ref{fig:kernels}c) &  \textbf{44.4 $\pm$ 0.3}   \\
      OOCS$_1$\hspace*{1ex}(kernel size 3x3, CS ratio 1/2) &  43.4 $\pm$ 0.5  \\
      OOCS$_2$\hspace*{1ex}(OOCS weights $1/n_c$ and $1/n_s$) &  43.6 $\pm$ 0.6  \\
      OOCS$_3$\hspace*{1ex}(OOCS computed after first pool) &  44.1 $\pm$ 0.4  \\
      OOCS$_4$\hspace*{1ex}(OOCS with trainable kernels) &  \textbf{44.5 $\pm$ 0.9}  \\
    \bottomrule
  \end{tabular}
  \end{adjustbox}
  \label{Imagenet_results-table}
  \vspace*{-2ex}
\end{table}

We implemented all models in TensorFlow 2.3 \cite{199317}, and used Adam for optimization \cite{DBLP:journals/corr/KingmaB14} with a learning rate of $10^{-4}$. We repeated all of the experiments for six times and report the mean value of the obtained results, together with their standard deviation. The latter was computed with a confidence level of 99.9\%.
\label{headings}

\begin{figure*}[t]
    \centering
    \includegraphics[width=0.7\textwidth]{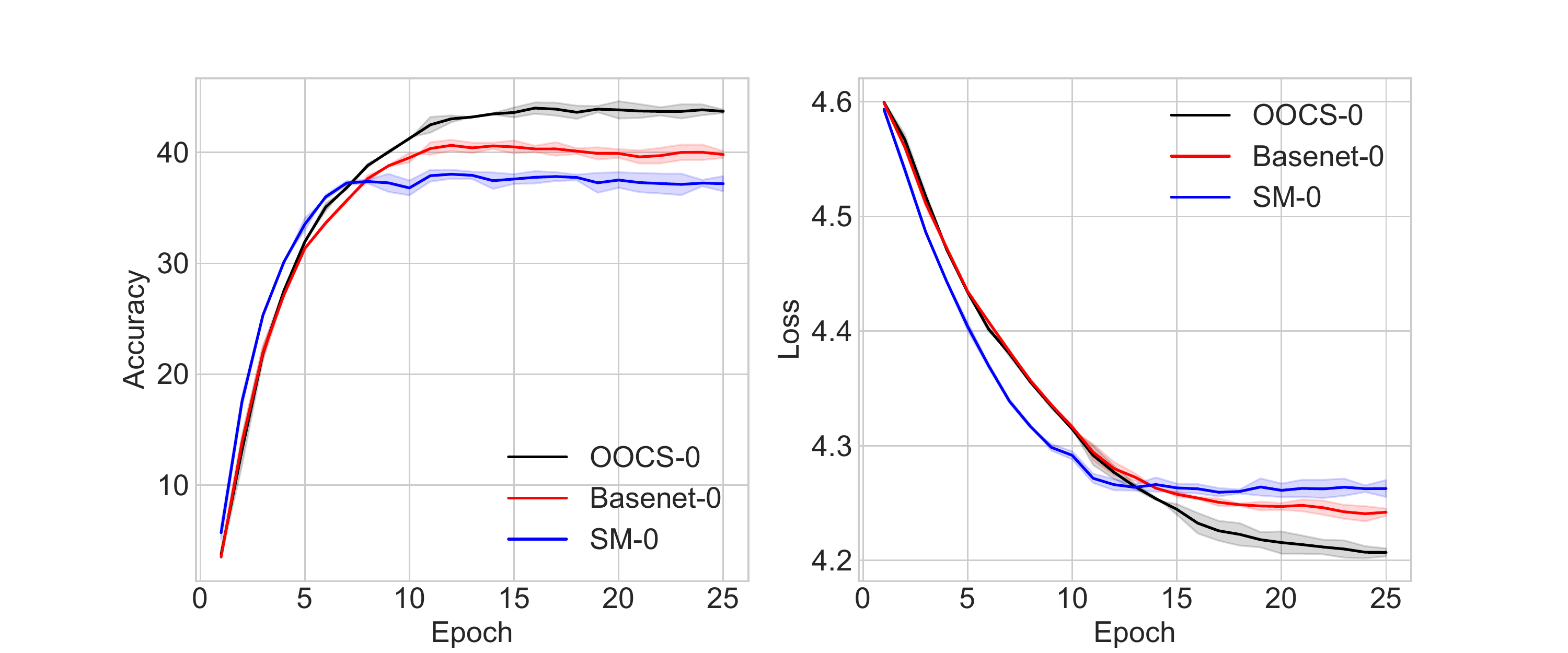}
    \caption{Validation accuracy and loss of the Basenet$_0$, OOCS$_0$, and SM-CNN$_0$  across epochs.}
    \label{fig:valtrain}
\end{figure*}

\subsection{Image Classification}
\label{subsec:Classification}
In this section, we assess the performance (accuracy) of OOCS models in image classification tasks.
%\todoforramin{Subsection should never immediately follow a section without at least one paragraph in between}

\textbf{Dataset.} We used a subset of the Imagenet~\cite{DBLP:conf/cvpr/DengDSLL009}. We randomly chose 600 samples from 100 categories. From these samples, we used 500 of each class for the training set, and 50 for each of the validation and test sets.  After cropping all the rectangular images to squares around the center, we resized all the images to $192\times192$ pixels. We did not perform any preprocessing on the images.

\textbf{Architectural details.}
%We compare OOCS to many standard and advanced models that are closely related to ours. In particular, 
We use three main architectures for this experiment as shown in Figure~\ref{fig:architecture}: Basenet$_{0}$, SM-CNN$_{0}$, and OOCS$_{0}$. We start from a standard deep CNN architecture for image recognition that we call Basenet$_{0}$. It consists of seven convolutional layers, five max-pooling layers, and two fully connected layers. All convolutional layers have $3\times3$ kernels and strides of 1.  We initialized all kernels with the He initialization~\cite{He_2015}. The max-pooling layers have $2\times2$ kernels and strides of 2. Dropouts~\cite{JMLR:v15:srivastava14a} follow the fully connected layers, and the last layer predicts the category with a softmax function. All of the hidden-weight layers have ReLU activation functions.
We have also three other variations Basenet$_{1}$, Basenet$_{2}$ and Basent$_{3}$ discussed in the "what happens" paragraph below.

We use the Surround-Modulation model of~\cite{NIPS2019_9719} for SM-CNN$_{0}$, with the kernel given in Figure~\ref{fig:3dplots}(b). Although SM and OO retinal receptive fields are two distinct biological phenomena, there are similarities between them and the way they are modeled. To ensure fairness, for the SM-CNN, we added the SM kernel to the activation maps of the first convolution layer in our baseline CNN, precisely aligned with what the authors suggested, and used the kernel weights they have reported. The variants SM-CNN$_{1-2}$ are described below in "what happens"..

For the OOCS-CNN network OOCS$_{0}$, we split the convolutional layers between the first and second pooling layers, into two parallel pathways with half of the filters of the original layers in each. The on-response is calculated by convolution with the on-kernel given in Figure~\ref{fig:3dplots}(c) and the off-response with its complementary kernel. The weights of these kernels are specified as non-trainable parameters.

\textbf{Performance.}
We compared the test accuracy of OOCS$_0$ with the performance of the Basenet$_0$ and of SM-CNN$_0$. Figure~\ref{fig:valtrain} shows the validation accuracy and loss for the three main network variants. As one can see, OOCS$_0$ outperforms the others by a large margin. Moreover, one can observe that SM-CN$_0$ has a better sample efficiency as it is stated in~\cite{NIPS2019_9719}, but has a poor performance on the validation set in comparison to the Basenet$_0$.

\begin{figure*}[t]
\centering
   \includegraphics[width=0.7\textwidth]{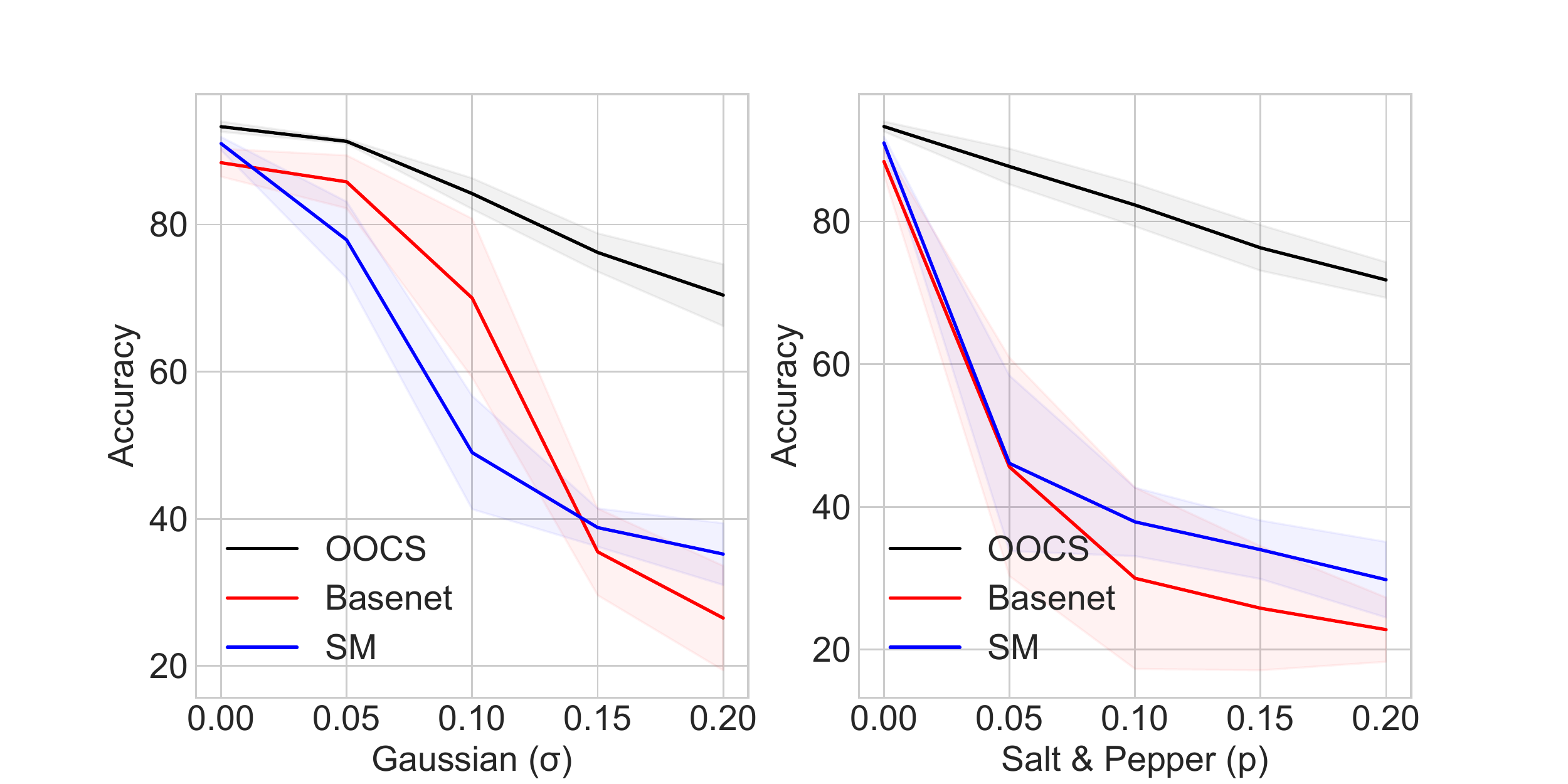}
   %\vspace*{-1ex}
   \caption{Noise-robustness. Test accuracy of the OOCS-CNN, the SM-CNN and the Basenet as a function of increasing input noise variance, with illumination Light0, and with n=6}
 %  Robustness evaluation of the OOCS-CNN, the Baseline CNN, the CNN with dropouts, and the SM-CNN on the Norb dataset in~\cite{DBLP:conf/cvpr/LeCunHB04}.
    \label{fig:noise_results}
    \vspace*{-4ex}
\end{figure*}

\textbf{What happens if we alter the architectures?} To more thoroughly evaluate OOCS performance, we designed eleven additional network variants: Basenet$_{1-5}$, SM-CNN$_{1-2}$, and OOCS$_{1-4}$. Basenet$_{1-2}$ were designed by adding an extra layer to Basenet$_0$, to see whether the performance of OOCS$_0$ can be obtained by increasing the number of training parameters of Basenet$_0$. This extra layer was added between the first and second pooling layers, with 64 filters. This new layer has a ReLU activation function in Basenet$_1$, and no activation function in Basenet$_2$. Basenet$_{3-4}$, have a skip-connection from the input to the output of the third convolution layer of Basenet$_0$, thus constructing a residual block~\cite{he2016deep}. This investigates if the performance of OOCS$_0$ is in fact related to the skip connection between the input and on-off pathways. The skip connection in Basenet$_4$ is a convolution layer with $5\times{}5$ kernel. Finally, Basenet$_5$ has the exact same architecture of OOCS$_0$, with the parallel pathways but with trainable filters alongside the skip connections. The kernels are initialised with the He initialization~\cite{He_2015}. 

SM-CNN$_1$ uses the kernel 
given in Figure~\ref{fig:kernels}(b), which is calculated from the DoG proposed by~\cite{NIPS2019_9719}, in Equation~\eqref{eq:rodiek}. We also designed an extra variant, called SM-CNN$_2$, which uses the OOCS kernel for the on-center-surround DoG, and which is given in Figure~\ref{fig:kernels}(c). 

OOCS$_1$ uses a smaller kernel to perform the convolutions: $3\times{}3$ as the size for the receptive field, and $\gamma\,{=}\,1/2$ as the center-surround ratio. OOCS$_2$ uses $1/n_c$ and $1/n_s$ as the weights for the elements in the center and surround in the kernel matrix, respectively, where $n_c$ is the number of central elements and $n_s$ is the number of elements in the surround. The aim was to indicate the effectiveness of the proposed DoG model. OOCS$_3$ computes the on- and off-responses from the inputs to the first layers of the pathways and adds them to the activation maps of the first layers in the pathways. Finally, OOCS$_4$ has the same architecture of OOCS$_0$, but with trainable filters on the On and Off skip connections, initialised with the On and Off kernels. Table~\ref{Imagenet_results-table} shows the top-1 test accuracy for the main and control networks. This model was added to see whether the On and Off kernels structure will be reserved throughout the training or we will eventually converge to other settings. All $64$ filters kept the OOCS structure, with differences in on and off weights intensities. OOCS$_0$ and it's variants OOCS$_{1-4}$  outperform all other models.

The results suggest that OOCS$_0$ performance can neither be achieved by increasing the capacity of Basenet$_0$ nor by adding skip-connections to Basenet$_0$. Moreover, the results for Basenet$_4-5$ show that the On-Off Center-Surround structure can not be learned by network without initialising or fixing the kernels with the calculated On and Off kernels. The On Center-Surround kernel of OOCS$_0$ also increases the performance of SM-CNN$_0$ considerably, and outperforms Basenet$_0$. Finally, SM-CNN$_1$ has a lower accuracy than Basenet$_0$.

The accuracy of OOCS$_1$ shows that with a smaller kernel, one still outperforms the Basenets and the SM-CNNs, but it decreases the accuracy compared to Basenet$_0$. OOCS$_2$ achieved a higher accuracy than these networks, too. This shows that forcing the positive and negative parts of the kernel to sum up to 1 and -1 respectively plays an important role in the OOCS-CNN's superior performance. Finally, OOCS$_3$ also outperforms the Basenets and SM-CNNs.

\textbf{OOCS ResNet.}
Here we modify the architecture of a residual network\cite{he2016deep} to include OOCS and evaluate the performance of it with and without the OOCS filters.

\textbf{Dataset.}
In this experiment, we use the same subset of Imagenet. We augmented the training images with randomly rotating them in the range of $15$, randomly shifting horizontally and vertically in the range of 10 percent of total width and height, and randomly flipping the images horizontally. 
\begin{table}[t]
  \centering
    \caption{Top-1 test accuracy and associated variance of the Resnet34 and Resnet34-OOCS models on Imagenet. n=6}
  \begin{tabular}{ll}
    \toprule
    \textbf{Models}  & \textbf{Acc} \\
    \midrule
     ResNet34 & 61.73 $\pm$ 0.6 \\
     ResNet34-OOCS &  \textbf{63.39 $\pm$ 0.7} \\
    \bottomrule
  \end{tabular}
  \label{Imagenet_resnet_results-table}
\end{table}

\textbf{Architectural details and Performance.}
We used the ResNet34\cite{he2016deep} as the base network. In order to add OOCS to it, we splitted the first two layers of the first residual block into the On and Off pathways with half of the filters of the original layers in each. The On and Off responses were calculated from the inputs to the first residual block. 

As Table~\ref{Imagenet_resnet_results-table} shows the results for both networks, OOCS enhances the performance of larger networks like ResNets and does not lose its effectiveness with data augmentation.

\begin{table*}[t]
  \centering
  \caption{Test Accuracy on all 6 lighting conditions for networks trained on Light0. n=6.}
  \vspace*{1.5ex}
  \begin{tabular}{lcccccc}
    \toprule
    \multicolumn{1}{c}{} & \multicolumn{1}{c}{Familiar Test Set} & \multicolumn{5}{c}{Unfamiliar Test Set}                   \\
   \midrule
    \textbf{Models}     & Light0     & Light1   & Light2    & Light3    & Light4    & Light5 \\
    \midrule
    Basenet & 88.4 $\pm$ 1.9  & 82.2 $\pm$ 2.4  & 46.4 $\pm$ 4.1 & 86.2 $\pm$ 1.5  & 58.8 $\pm$ 2.7  & 79.5 $\pm$ 4.9     \\
    Basenet-L2 & 89.6 $\pm$ 1.9  & 87.0 $\pm$ 2.8  & 48.5 $\pm$ 3.8  & 86.9 $\pm$ 1.5  & 60.8 $\pm$ 3.5 & 80.3 $\pm$ 4.5      \\
    Basenet-Dropout & 89.2 $\pm$ 2.0  & 86.3 $\pm$ 5.7 & 48.5 $\pm$ 4.1  & 86.3 $\pm$ 1.2 &  59.9 $\pm$ 3.8 & 81.3 $\pm$ 5.4      \\
    Basenet-BN     & 89.2 $\pm$ 1.4  & 84.0 $\pm$ 3.7 & 37.4 $\pm$ 7.7 & 87.3 $\pm$ 2.0  & 56.2 $\pm$ 3.7  & 78.9 $\pm$ 4.4   \\
    SM-CNN     & 91.0 $\pm$ 0.9  & 90.6 $\pm$ 0.6  & \textbf{62.5 $\pm$ 1.5}  & 87.3 $\pm$ 1.6  & 61.8 $\pm$ 2.3 & 81.6 $\pm$ 1.7 \\
    OOCS-CNN (Ours)    & \textbf{93.3 $\pm$ 0.7} & \textbf{90.9 $\pm$ 0.7} & 56.7 $\pm$ 1.6 & \textbf{91.2 $\pm$ 1.0} & \textbf{61.9 $\pm$ 1.5} & \textbf{88.2 $\pm$ 1.0}  \\
    \bottomrule
  \end{tabular}
  \label{light0-table}
  \vspace*{-2ex}
\end{table*}
\subsection{Robustness Evaluation}
\label{subsec:Robustness}
Here we evaluate OOCS robustness. To this end we extensively study how deep CNN models perform under the variation of lighting conditions and to distribution shifts. 

\textbf{Dataset.} We used the Norb dataset~\cite{DBLP:conf/cvpr/LeCunHB04} to assess the robustness of our proposed OOCS-CNN architecture. This dataset contains images of 3D objects belonging to five generic categories. The images are of size $96\,{\times}\,96$ pixels, and photographed under six different lighting conditions. In a first experiment, we trained our networks on the images from one lighting condition, Light0, and then tested them on all six lighting conditions. In a second experiment we added different kinds of noise to the testing images, in order to evaluate the robustness in the presence of noise.
%\begin{figure*}
%   \includegraphics[width=1.0\textwidth]{Images/%norbs.pdf}
%   \caption{}
%    \label{fig:norbs}
%\end{figure*}

\textbf{Architectural details.}
Given the image size, we designed CNN Basenets with 6 convolutional layers, 4 max-pooling layers, and 2 fully connected layers. All convolutional layers have $3\times3$ kernels, and strides of 1. The max-pooling layers have $2\times2$ kernels and strides of 2. The number of filters in the first convolution layer is 32 and gets doubled after each pooling layer. For the OOCS-CNN, we split the layers between the first and second max-pooling to On and Off pathways and compute the On and Off convolutions from the input to the first layers of pathways. 

\textbf{Robustness Evaluation.}
To assess the robustness of OOCS-CNNs under changes in illumination, we tested the networks on images with a lighting condition that was different from the one in the training set. We compare the performance of our model with Basenet-L2, Basenet-Dropout and Basenet-BN, in addition to the Basenet and SM-CNN. %We added $L2$ regularization to all of the training parameters in Basenet-L2. In Basenet-Dropout, we added dropout~\cite{JMLR:v15:srivastava14a} to the fully connected layers, and in Basenet-BN, we have added Batch Normalization~\cite{DBLP:journals/corr/IoffeS15} to all of the convolution layers. 
\begin{table}[t]
  %\fontsize{9.0pt}{10.8pt} \selectfont
   
  \centering
  
  \caption{Test Accuracy and variance for images from Light0 with Gaussian noise for networks trained on Light0. n=6.}
  \vspace*{1.5ex}
  \begin{adjustbox}{width=1\columnwidth}
  \begin{tabular}{lcccc}
    \toprule
    \multicolumn{1}{c}{} & \multicolumn{4}{c}{\textbf{Gaussian Noise ($\sigma$)}} \\
    \midrule
    \textbf{Models} &0.05&0.1&0.15&0.2\\
    \midrule
    Basenet&$85.8_{\pm3.6}$&$70.0_{\pm10.8}$&$35.5_{\pm5.9}$&$26.5_{\pm7.1}$\\
    Basenet-L2&$84.8_{\pm2.0}$&$73.3_{\pm3.9}$&$46.0_{\pm10.4}$&$33.0_{\pm8.1}$\\
    Basenet-D&$87.5_{\pm2.9}$&$80.8_{\pm2.4}$&$66.5_{\pm11.7}$&$42.2_{\pm17.7}$\\
    Basenet-BN& $84.4_{\pm2.8}$&$69.8_{\pm10.3}$&$48.8_{\pm13.3}$&$38.4_{\pm8.3}$\\
    SM-CNN&$77.9_{\pm5.2}$&$49.0_{\pm7.7}$&$38.8_{\pm2.6}$ &$35.2_{\pm4.2}$\\
    OOCS-CNN&$\mathbf{91.3_{\pm0.3}}$&$\mathbf{84.2_{\pm2.1}}$&$\mathbf{76.2_{\pm2.6}}$&$\mathbf{70.4_{\pm4.0}}$ \\
    \bottomrule
  \end{tabular}
  \end{adjustbox}
  \label{light0noisegaussian-table}
\end{table}

%\FloatBarrier
\begin{table}[t]
  \fontsize{9.0pt}{10.8pt} \selectfont
  \caption{Test Accuracy and variance for images from Light0 with Salt\,\&\,Pepper noise for networks trained on Light0. n=10.}
  \vspace*{1.5ex}
  \centering
  \begin{adjustbox}{width=1\columnwidth}
  \begin{tabular}{lllll}
    \toprule
    \multicolumn{1}{c}{} & \multicolumn{4}{c}{\textbf{Salt and Pepper Noise ($p$)}} \\
    \midrule
    \textbf{Models}     & 0.05     & 0.1   & 0.15    & 0.2 \\
    \midrule
    Basenet & $45.6_{\pm15.3}$ & $30.0_{\pm12.7}$ &$25.8_{\pm8.7}$  & $22.8_{\pm4.5}$     \\
    Basenet-L2 & $49.7_{\pm9.4}$ & $31.6_{\pm7.5}$ & $26.0_{\pm5.4}$ & $23.7_{\pm3.5}$      \\
    Basenet-D & $71.5_{\pm7.5}$ & $48.5_{\pm13.3}$ & $29.1_{\pm12.1}$ & $22.1_{\pm3.5}$      \\
    Basenet-BN     & $54.2_{\pm14.7}$ & $40.6_{\pm13.0}$ & $34.2_{\pm10.8}$ & $29.7_{\pm8.3}$  \\
    SM-CNN  & $46.1_{\pm12.3}$ & $37.9_{\pm4.8}$ & $34.0_{\pm4.1}$ & $29.8_{\pm5.3}$ \\
    OOCS-CNN & $\mathbf{87.7_{\pm2.5}}$ & $\mathbf{82.3_{\pm3.0}}$ & $\mathbf{76.3_{\pm3.2}}$ & $\mathbf{71.8_{\pm2.5}}$  \\
    \bottomrule
  \end{tabular}
  \end{adjustbox}
  \label{light0noisesp-table}
\end{table}

Table~\ref{light0-table} shows the test-accuracy and associated variance for all baselines trained on Light0. Except for Light2, where SM-CNN proved to be superior to the other networks, the OOCS-CNN outperformed the other CNN variants with a considerable margin. 

We then added Gaussian and Salt $\&$ Pepper noise to the test set, to evaluate the performance of the networks under perturbations. Table~\ref{light0noisegaussian-table} and Table~\ref{light0noisesp-table} show the test accuracy and variance for networks, trained on Light0 and tested on the same lighting condition, but in the presence of noise. 
%\FloatBarrier

The results show that the OOCS-CNN is more robust than other regularization methods. Not only has it higher accuracy, the results of the OOCS-CNN have considerably lower variance in comparison to the other networks on modified test sets. 
Figure~\ref{fig:noise_results} shows the significant level of robustness to noise achieved by OOCS compared to other methods.

\textbf{Robustness to distribution shifts.}
As discussed before, on- and off-convolutions extract different features for objects on light and dark backgrounds. Both are needed to capture higher details. To assess our claim more thoroughly, we designed the following experiment. 

\textbf{Dataset.}
We used the MNIST dataset \cite{lecun-mnisthandwrittendigit-2010} to train our networks. The original dataset contains black handwritten digits on white background. We trained our networks on the original images, but for the test set we altered all the pixel values by subtracting them from $255.$ Figure \ref{fig:mnist} shows the original/altered image for one sample.
\begin{figure}[t]
\centering
   \includegraphics[width=0.27\textwidth]{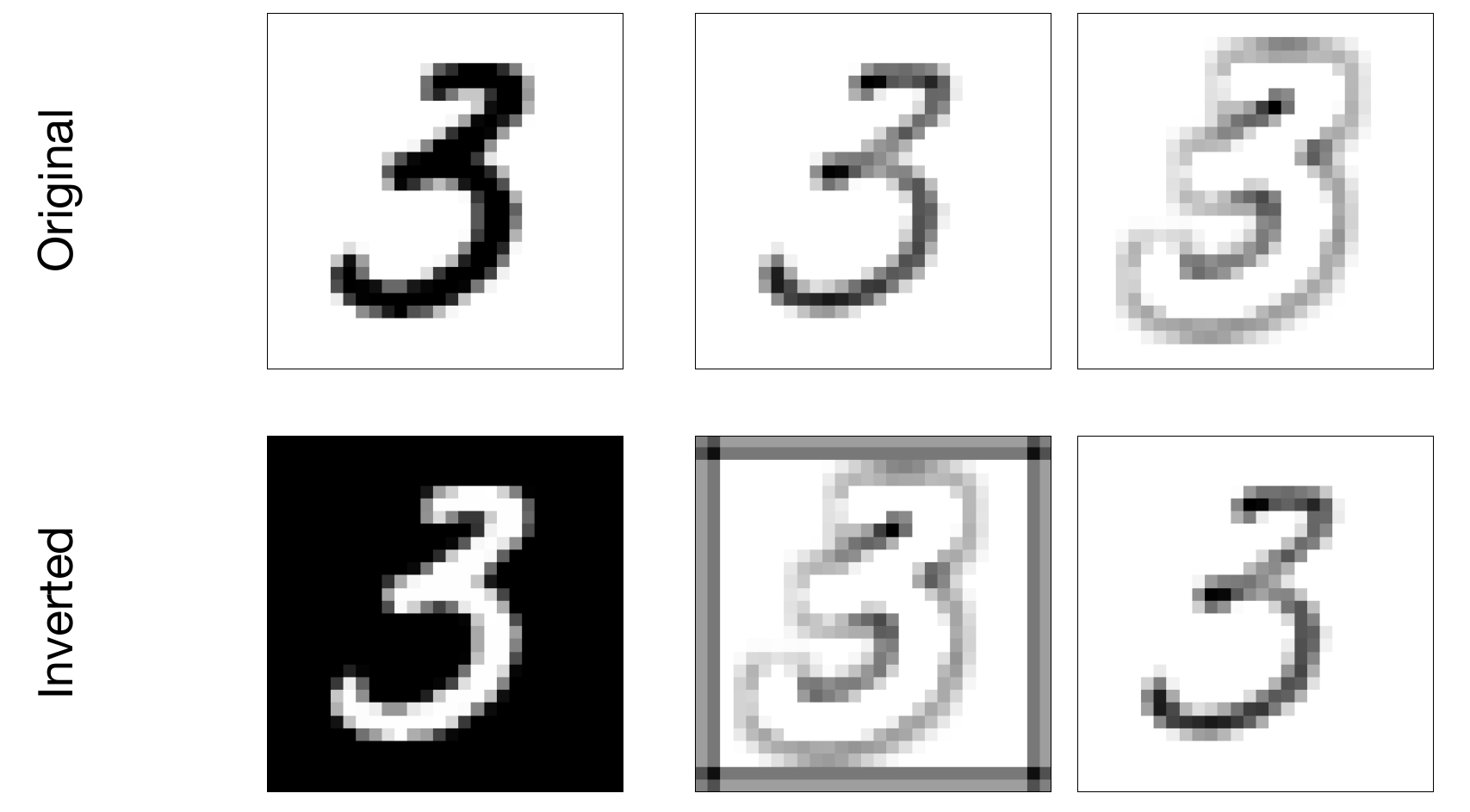}
   \caption{The on- and off-convolutions are shown in the second and the third columns, respectively, for a sample of the MNIST dataset. In the first row we show the original image, and in the second row the inverted image, with all pixels subtracted from 255.}
 %  Robustness evaluation of the OOCS-CNN, the Baseline CNN, the CNN with dropouts, and the SM-CNN on the Norb dataset in~\cite{DBLP:conf/cvpr/LeCunHB04}.
    \label{fig:mnist}
%\vspace*{-2ex}
\end{figure}

\begin{table}[t]
  \centering
    \caption{Test Accuracy of the networks on the original black-on-white and on the inverted white-on-black test sets. n=6.}
  \vspace*{1.5ex}
  \begin{tabular}{llll}
    \toprule
         & Basenet     & SM-CNN   & OOCS-CNN \\
    \midrule
      Original & $99.1_{\pm0.2}$ & $99.2_{\pm0.1}$ & $99.1_{\pm0.2}$   \\
      Inverted & $29.0_{\pm10.1}$ & $35.7_{\pm4.9}$ & $\mathbf{93.9_{\pm1.2}}$   \\
    \bottomrule
  \end{tabular}
  \label{mnist_table}
\end{table}

%\textbf{Architectural details.}
%For the Basenet we used a CNN with four convolution layers, two max-pooling layers, and one fully connected layer. All of the convolution layers have $3\times3$ kernels, and strides of 1. The max-pooling layers have $2\times2$ kernels and strides of 2. The number of filters in the first convolution layer is 32 and gets doubled after first pooling layer. For the OOCS-CNN, we apply the On and Off convolutions to the input images as a preprocessing step and feed them to the Basenet CNN.

\textbf{Robustness evaluation.}
Table \ref{mnist_table} shows the test accuracy of the networks on the original and altered  MNIST test images. On the test set with inverted colors, the accuracy of the Basenet and SM-CNN drops sharply. The OOCS-CNN on the other hand, achieves a much higher accuracy on this challenging test set. Figure~\ref{fig:mnist} shows that on the samples with black foreground, the on-center convolution extracts features similar to the ones extracted by the off-center convolution on the samples with black background and vice versa. Being equipped with both in OOCS-CNN is the reason it still performs better on the altered test set.

\section{Discussion, Scope and Conclusions}
\label{sec:Discussion}
Inspired by the retinal ganglial cells in vertebrates, we first proposed an on-off center-surround (OOCS) enhancement to the receptive fields of vision networks. We then showed that the OOOCS pathways impose an inductive bias on the vision networks, which enhances their robustness to the variation of lighting conditions. We took advantage of the studies in the field of Neuroscience and used an improved CS kernel as a ubiquitous block for obtaining more accurate and more robust vision based networks. The OOCS addition to a CNN is easy to implement, does not increase the number of trainable parameters, and it increases its performance.

\textbf{Performance Out of IID Setting.} OOCS performs extremely well under test-set distribution-shifts. This is a direct result of the complementary on- and off-pathways.

\textbf{Inductive Biases vs. Regularization.} Our experiments (Table 2) show that standard regularization methods such as L2, Dropout and batch-normalization are less effective in improving generalization under distribution shifts, compared to the OOCS. Figure~\ref{fig:reg_vs_ind} gives the absolute distance of the familiar test error (of the same lighting condition as the training set), to the test error of different lighting conditions. We observe that OOCS has a lower distance for distribution shifts compared to regularization methods.
\begin{figure}[t]
    \centering
    \includegraphics[width=0.35\textwidth]{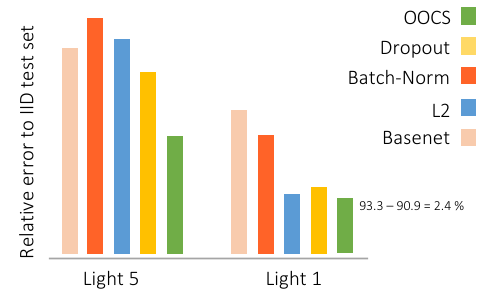}
    \vspace{-3mm}
    \caption{Relative test error for different lighting conditions compared to the IID test error (shorter bar is better). The values are taken from Table \ref{light0-table}. Each bar represents the deviation of the test error from the test error for Light0.}
    \label{fig:reg_vs_ind}
    \vspace*{-4ex}
\end{figure}

\textbf{Robustness to Digital Distribution Shifts.} What happens if we add OOCS filters to a ResNet architecture? would this improve their robustness properties to distribution shifts such as Gaussian perturbations, brightness, and contrast. We studied this systematically in a set of experiments where we added OOCS filters to a ResNet-34, and compared its performance under variation of these perturbations (See Figures S3 and S4 in the supplementary materials). We observed consistent improvements achieved when OOCS filters are deployed. 

\textbf{Usability and Impact of OOCS Blocks.} OOCS pathways can be added to any vision-based model regardless of their exact architecture. They can be interpreted as targeted residual blocks explicitly designed to perform specialized edge detection. In decision critical applications \cite{lechner2019designing,lechner2020neural}, object recognition under different lighting conditions is a game changer. For instance vision-based self-driving cars are sensitive to lighting conditions \cite{lechner2021adversarial,lechner2020gershgorin}. OOCS can robustify driving under direct sun or going from shaded to well lit regions.

\textbf{Bio-inspirations and Future Work.} In this work, our main focus was to bring insights from a simplified model of the retinal cells to designing more robust image classification modules. There are additional top-down connections between layers in the retina, and the sizes of the receptive fields change depending on their location. It is therefore worth exploring the implementation of a complete model of the retina and to further improving the performance of this model, similar to many works that aim to transform biological mechanisms \cite{lechner2017worm,sarma2018openworm,gleeson2018c302,dabney2020distributional} into better machine learning models \cite{hasani2020natural}. The OOCS helps a network to extract the inner and outer edges, hence more work can be done to investigate using it as a preprocessing or data augmentation tool, especially for medical image segmentation tasks.

\section*{Acknowledgements}
Z.B. is supported by the Doctoral College Resilient Embedded Systems, which is run jointly by the TU Wien's Faculty of Informatics and the UAS Technikum Wien. R.G. is partially supported by the Horizon 2020 Era-Permed project Persorad, and ECSEL Project grant no. 783163 (iDev40). R.H and D.R were partially supported by Boeing and MIT. M.L. is supported in part by the Austrian Science Fund (FWF) under grant Z211-N23 (Wittgenstein
Award).

\nocite{Strisciuglio2019}
\nocite{Azzopardi2014}
%\bibliography{references}

\begin{thebibliography}{54}
\providecommand{\natexlab}[1]{#1}
\providecommand{\url}[1]{\texttt{#1}}
\expandafter\ifx\csname urlstyle\endcsname\relax
  \providecommand{\doi}[1]{doi: #1}\else
  \providecommand{\doi}{doi: \begingroup \urlstyle{rm}\Url}\fi

\bibitem[Abadi et~al.(2016)Abadi, Barham, Chen, Chen, Davis, Dean, Devin,
  Ghemawat, Irving, Isard, Kudlur, Levenberg, Monga, Moore, Murray, Steiner,
  Tucker, Vasudevan, Warden, Wicke, Yu, and Zheng]{199317}
Abadi, M., Barham, P., Chen, J., Chen, Z., Davis, A., Dean, J., Devin, M.,
  Ghemawat, S., Irving, G., Isard, M., Kudlur, M., Levenberg, J., Monga, R.,
  Moore, S., Murray, D., Steiner, B., Tucker, P., Vasudevan, V., Warden, P.,
  Wicke, M., Yu, Y., and Zheng, X.
\newblock Tensorflow: {A} {S}ystem for {L}arge-{S}cale {M}achine {L}earning.
\newblock In \emph{12th {USENIX} Symposium on Operating Systems Design and
  Implementation ({OSDI} 16)}, pp.\  265--283, Savannah, GA, November 2016.
  {USENIX}.

\bibitem[Allman et~al.(1985)Allman, Miezin, and
  McGuinness]{doi:10.1146/annurev.ne.08.030185.002203}
Allman, J., Miezin, F., and McGuinness, E.
\newblock Stimulus specific responses from beyond the classical receptive
  field: Neurophysiological mechanisms for local-global comparisons in visual
  neurons.
\newblock \emph{Annual Review of Neuroscience}, 8\penalty0 (1):\penalty0
  407--430, 1985.

\bibitem[Azzopardi et~al.(2014)Azzopardi, Rodríguez-Sánchez, Piater, and
  Petkov]{Azzopardi2014}
Azzopardi, G., Rodríguez-Sánchez, A., Piater, J., and Petkov, N.
\newblock A push-pull corf model of a simple cell with antiphase inhibition
  improves snr and contour detection.
\newblock \emph{PLOS ONE}, 9\penalty0 (7):\penalty0 1--13, 07 2014.
\newblock \doi{10.1371/journal.pone.0098424}.

\bibitem[Blackburn(1993)]{Blackburn1993ASC}
Blackburn, M.
\newblock A {S}imple {C}omputational {M}odel of {C}enter-{S}urround {R}eceptive
  {F}ields in the {R}etina.
\newblock Technical Report 2454, Ocean Surveillance Center, Feb 1993.

\bibitem[Callaway(2005)]{doi:10.1113/jphysiol.2005.088047}
Callaway, E.
\newblock Structure and {F}unction of {P}arallel {P}athways in the {P}rimate
  {E}arly {V}isual {S}ystem.
\newblock \emph{The Journal of Physiology}, 566\penalty0 (1):\penalty0 13--19,
  2005.

\bibitem[Ciresan et~al.(2011)Ciresan, Meier, Masci, Gambardella, and
  Schmidhuber]{inproceedings}
Ciresan, D., Meier, U., Masci, J., Gambardella, L., and Schmidhuber, J.
\newblock Flexible, {H}igh {P}erformance {C}onvolutional {N}eural {N}etworks
  for {I}mage {C}lassification.
\newblock \emph{International Joint Conference on Artificial Intelligence
  IJCAI-2011}, pp.\  1237--1242, 07 2011.

\bibitem[Dabney et~al.(2020)Dabney, Kurth-Nelson, Uchida, Starkweather,
  Hassabis, Munos, and Botvinick]{dabney2020distributional}
Dabney, W., Kurth-Nelson, Z., Uchida, N., Starkweather, C.~K., Hassabis, D.,
  Munos, R., and Botvinick, M.
\newblock A distributional code for value in dopamine-based reinforcement
  learning.
\newblock \emph{Nature}, 577\penalty0 (7792):\penalty0 671--675, 2020.

\bibitem[Dacey(2004)]{Dacey2004-DACOO}
Dacey, D.
\newblock 20 {O}rigins of {P}erception: {R}etinal {G}anglion {C}ell {D}iversity
  and the {C}reation of {P}arallel {V}isual {P}athways.
\newblock In \emph{The Cognitive Neurosciences}, pp.\  281. MIT Press, 2004.

\bibitem[Deng et~al.(2009)Deng, Dong, Socher, Li, Li, and
  Li]{DBLP:conf/cvpr/DengDSLL009}
Deng, J., Dong, W., Socher, R., Li, L., Li, K., and Li, F.
\newblock Imagenet: {A} {L}arge-{S}cale {H}ierarchical {I}mage {D}atabase.
\newblock In \emph{{IEEE} Computer Society Conference on Computer Vision and
  Pattern Recognition}, pp.\  248--255, Miami, Florida, USA, June 2009. {IEEE}
  Computer Society.

\bibitem[Diederik \& Ba(2015)Diederik and Ba]{DBLP:journals/corr/KingmaB14}
Diederik, P. and Ba, J.
\newblock Adam: {A} {M}ethod for {S}tochastic {O}ptimization.
\newblock In Bengio, Y. and LeCun, Y. (eds.), \emph{3rd International
  Conference on Learning Representations}, San Diego, CA, USA, May 2015.

\bibitem[Ding et~al.(2019)Ding, Guo, Ding, and Han]{ding2019acnet}
Ding, X., Guo, Y., Ding, G., and Han, J.
\newblock Acnet: Strengthening the kernel skeletons for powerful cnn via
  asymmetric convolution blocks.
\newblock In \emph{Proceedings of the IEEE/CVF International Conference on
  Computer Vision}, pp.\  1911--1920, 2019.

\bibitem[Enroth-Cugell \& Pinto(1972)Enroth-Cugell and
  Pinto]{enroth-cugell_pinto_1972}
Enroth-Cugell, C. and Pinto, L.~H.
\newblock Properties of the {S}urround {R}esponse {M}echanism of {C}at
  {R}etinal {G}anglion {C}ells and {C}entre-{S}urround {I}nteraction.
\newblock \emph{The Journal of Physiology}, 220\penalty0 (2):\penalty0
  403–439, Jan 1972.

\bibitem[Fukushima(1980)]{fukushima_1980}
Fukushima, K.
\newblock Neocognitron: {A} {S}elf-{O}rganizing {N}eural {N}etwork {M}odel for
  a {M}echanism of {P}attern {R}ecognition {U}naffected by {S}hift in
  {P}osition.
\newblock \emph{Biological Cybernetics}, 36\penalty0 (4):\penalty0 193–202,
  1980.

\bibitem[Fukushima(2003)]{fukushima_2003}
Fukushima, K.
\newblock Neocognitron for {H}andwritten {D}igit {R}ecognition.
\newblock \emph{Journal of Neurocomputing}, 51:\penalty0 161–180, 2003.

\bibitem[Gleeson et~al.(2018)Gleeson, Lung, Grosu, Hasani, and
  Larson]{gleeson2018c302}
Gleeson, P., Lung, D., Grosu, R., Hasani, R., and Larson, S.~D.
\newblock c302: a multiscale framework for modelling the nervous system of
  caenorhabditis elegans.
\newblock \emph{Philosophical Transactions of the Royal Society B: Biological
  Sciences}, 373\penalty0 (1758):\penalty0 20170379, 2018.

\bibitem[Hartline(1940)]{hartline_1940}
Hartline, H.~K.
\newblock The receptive fields of optic nerve fibers.
\newblock \emph{American Journal of Physiology-Legacy Content}, 130\penalty0
  (4):\penalty0 690–699, 1940.

\bibitem[Hasani et~al.(2019)Hasani, Soleymani, and Aghajan]{NIPS2019_9719}
Hasani, H., Soleymani, M., and Aghajan, H.
\newblock Surround {M}odulation: {A} {B}io-inspired {C}onnectivity {S}tructure
  for {C}onvolutional {N}eural {N}etworks.
\newblock In Wallach, H., Larochelle, H., Beygelzimer, A., d~Alch\'{e}-Buc, F.,
  Fox, E., and Garnett, R. (eds.), \emph{Advances in Neural Information
  Processing Systems 32}, pp.\  15903--15914. Curran Associates, Inc., 2019.

\bibitem[Hasani et~al.(2020)Hasani, Lechner, Amini, Rus, and
  Grosu]{hasani2020natural}
Hasani, R., Lechner, M., Amini, A., Rus, D., and Grosu, R.
\newblock A natural lottery ticket winner: Reinforcement learning with ordinary
  neural circuits.
\newblock In \emph{International Conference on Machine Learning}, pp.\
  4082--4093. PMLR, 2020.

\bibitem[Hasani et~al.(2021)Hasani, Lechner, Amini, Rus, and
  Grosu]{Hasani2021liquid}
Hasani, R., Lechner, M., Amini, A., Rus, D., and Grosu, R.
\newblock Liquid time-constant networks.
\newblock \emph{Proceedings of the AAAI Conference on Artificial Intelligence},
  35\penalty0 (9):\penalty0 7657--7666, May 2021.

\bibitem[He et~al.(2015)He, Zhang, Ren, and Sun]{He_2015}
He, K., Zhang, X., Ren, S., and Sun, J.
\newblock Delving {D}eep into {R}ectifiers: {S}urpassing {H}uman-{L}evel
  {P}erformance on {I}magenet {C}lassification.
\newblock \emph{2015 IEEE International Conference on Computer Vision}, Dec
  2015.
\newblock \doi{10.1109/iccv.2015.123}.
\newblock URL \url{http://dx.doi.org/10.1109/ICCV.2015.123}.

\bibitem[He et~al.(2016)He, Zhang, Ren, and Sun]{he2016deep}
He, K., Zhang, X., Ren, S., and Sun, J.
\newblock Deep residual learning for image recognition.
\newblock In \emph{Proceedings of the IEEE conference on computer vision and
  pattern recognition}, pp.\  770--778, 2016.

\bibitem[Hu et~al.(2018)Hu, Shen, and Sun]{hu2018squeeze}
Hu, J., Shen, L., and Sun, G.
\newblock Squeeze-and-excitation networks.
\newblock In \emph{Proceedings of the IEEE conference on computer vision and
  pattern recognition}, pp.\  7132--7141, 2018.

\bibitem[Hubel \& Wiesel(1965)Hubel and Wiesel]{doi:10.1152/jn.1965.28.2.229}
Hubel, D. and Wiesel, T.
\newblock Receptive fields and functional architecture in two nonstriate visual
  areas (18 and 19) of the cat.
\newblock \emph{Journal of Neurophysiology}, 28\penalty0 (2):\penalty0
  229--289, 1965.

\bibitem[Hubel \& Wiesel(1968)Hubel and Wiesel]{hubelWiesel1968}
Hubel, D.~H. and Wiesel, T.~N.
\newblock Receptive fields and functional architecture of monkey striate
  cortex.
\newblock \emph{The Journal of Physiology}, 195\penalty0 (1):\penalty0
  215--243, 1968.

\bibitem[Jacobsen et~al.(2016)Jacobsen, Van~Gemert, Lou, and
  Smeulders]{jacobsen2016structured}
Jacobsen, J.-H., Van~Gemert, J., Lou, Z., and Smeulders, A.~W.
\newblock Structured receptive fields in cnns.
\newblock In \emph{Proceedings of the IEEE Conference on Computer Vision and
  Pattern Recognition}, pp.\  2610--2619, 2016.

\bibitem[Kandel et~al.(2013)Kandel, Jessell, Schwartz, Siegelbaum, and
  Hudspeth]{mack2013principles}
Kandel, E., Jessell, T., Schwartz, J., Siegelbaum, S., and Hudspeth, A.
\newblock \emph{Principles of Neural Science}.
\newblock Fifth Edition. McGraw-Hill Medical / Education, 2013.

\bibitem[Kim et~al.(2016)Kim, Sangjun, Kim, and Lee]{kim2016convolutional}
Kim, J., Sangjun, O., Kim, Y., and Lee, M.
\newblock Convolutional neural network with biologically inspired retinal
  structure.
\newblock \emph{Procedia Computer Science}, 88:\penalty0 145--154, 2016.

\bibitem[Knierim \& van Essen(1992)Knierim and van
  Essen]{doi:10.1152/jn.1992.67.4.961}
Knierim, J. and van Essen, D.
\newblock Neuronal responses to static texture patterns in area v1 of the alert
  macaque monkey.
\newblock \emph{Journal of Neurophysiology}, 67\penalty0 (4):\penalty0
  961--980, 1992.

\bibitem[Kruizinga \& Petkov(2000)Kruizinga and Petkov]{kruizinga_petkov_2000}
Kruizinga, P. and Petkov, N.
\newblock Computational {M}odel of {D}ot-{P}attern {S}elective {C}ells.
\newblock \emph{Biological Cybernetics}, 83\penalty0 (4):\penalty0 313–325,
  Jun 2000.

\bibitem[Kuffler(1953)]{doi:10.1152/jn.1953.16.1.37}
Kuffler, S.
\newblock Discharge {P}atterns and {F}unctional {O}rganization of {M}ammalian
  {R}etina.
\newblock \emph{Journal of Neurophysiology}, 16\penalty0 (1):\penalty0 37--68,
  1953.

\bibitem[Laskar et~al.(2018)Laskar, Giraldo, and
  Schwartz]{laskar2018correspondence}
Laskar, M. N.~U., Giraldo, L. G.~S., and Schwartz, O.
\newblock Correspondence of deep neural networks and the brain for visual
  textures.
\newblock \emph{arXiv preprint arXiv:1806.02888}, 2018.

\bibitem[Lechner et~al.(2017)Lechner, Grosu, and Hasani]{lechner2017worm}
Lechner, M., Grosu, R., and Hasani, R.~M.
\newblock Worm-level control through search-based reinforcement learning.
\newblock \emph{arXiv preprint arXiv:1711.03467}, 2017.

\bibitem[Lechner et~al.(2019)Lechner, Hasani, Zimmer, Henzinger, and
  Grosu]{lechner2019designing}
Lechner, M., Hasani, R., Zimmer, M., Henzinger, T.~A., and Grosu, R.
\newblock Designing worm-inspired neural networks for interpretable robotic
  control.
\newblock In \emph{2019 International Conference on Robotics and Automation
  (ICRA)}, pp.\  87--94. IEEE, 2019.

\bibitem[Lechner et~al.(2020{\natexlab{a}})Lechner, Hasani, Amini, Henzinger,
  Rus, and Grosu]{lechner2020neural}
Lechner, M., Hasani, R., Amini, A., Henzinger, T.~A., Rus, D., and Grosu, R.
\newblock Neural circuit policies enabling auditable autonomy.
\newblock \emph{Nature Machine Intelligence}, 2\penalty0 (10):\penalty0
  642--652, 2020{\natexlab{a}}.

\bibitem[Lechner et~al.(2020{\natexlab{b}})Lechner, Hasani, Rus, and
  Grosu]{lechner2020gershgorin}
Lechner, M., Hasani, R., Rus, D., and Grosu, R.
\newblock Gershgorin loss stabilizes the recurrent neural network compartment
  of an end-to-end robot learning scheme.
\newblock In \emph{2020 IEEE International Conference on Robotics and
  Automation (ICRA)}, pp.\  5446--5452. IEEE, 2020{\natexlab{b}}.

\bibitem[Lechner et~al.(2021)Lechner, Hasani, Grosu, Rus, and
  Henzinger]{lechner2021adversarial}
Lechner, M., Hasani, R., Grosu, R., Rus, D., and Henzinger, T.~A.
\newblock Adversarial training is not ready for robot learning.
\newblock \emph{arXiv preprint arXiv:2103.08187}, 2021.

\bibitem[LeCun \& Cortes(2010)LeCun and
  Cortes]{lecun-mnisthandwrittendigit-2010}
LeCun, Y. and Cortes, C.
\newblock {MNIST} handwritten digit database.
\newblock 2010.

\bibitem[LeCun et~al.(1989)LeCun, Boser, Denker, Henderson, Howard, Hubbard,
  and Jackel]{doi:10.1162/neco.1989.1.4.541}
LeCun, Y., Boser, B., Denker, J., Henderson, D., Howard, R., Hubbard, W., and
  Jackel, L.
\newblock Backpropagation applied to handwritten zip code recognition.
\newblock \emph{Neural Computation}, 1\penalty0 (4):\penalty0 541--551, 1989.

\bibitem[Lecun et~al.(1998)Lecun, Bottou, Bengio, and
  Haffner]{Lecun98gradient-basedlearning}
Lecun, Y., Bottou, L., Bengio, Y., and Haffner, P.
\newblock Gradient-based {L}earning {A}pplied to {D}ocument {R}ecognition.
\newblock In \emph{Proceedings of the IEEE}, pp.\  2278--2324, 1998.

\bibitem[LeCun et~al.(2004)LeCun, Huang, and Bottou]{DBLP:conf/cvpr/LeCunHB04}
LeCun, Y., Huang, F., and Bottou, L.
\newblock Learning {M}ethods for {G}eneric {O}bject {R}ecognition with
  {I}nvariance to {P}ose and {L}ighting.
\newblock In \emph{{IEEE} Computer Society Conference on Computer Vision and
  Pattern Recognition}, pp.\  97--104, Washington DC, USA, July 2004. IEEE.

\bibitem[Li et~al.(2019{\natexlab{a}})Li, Wang, Hu, and Yang]{li2019selective}
Li, X., Wang, W., Hu, X., and Yang, J.
\newblock Selective kernel networks.
\newblock In \emph{Proceedings of the IEEE/CVF Conference on Computer Vision
  and Pattern Recognition}, pp.\  510--519, 2019{\natexlab{a}}.

\bibitem[Li et~al.(2019{\natexlab{b}})Li, Chen, Wang, and Zhang]{li2019scale}
Li, Y., Chen, Y., Wang, N., and Zhang, Z.
\newblock Scale-aware trident networks for object detection.
\newblock In \emph{Proceedings of the IEEE/CVF International Conference on
  Computer Vision}, pp.\  6054--6063, 2019{\natexlab{b}}.

\bibitem[Linsley et~al.(2018)Linsley, Kim, Veerabadran, Windolf, and
  Serre]{NIPS2018_7300}
Linsley, D., Kim, J., Veerabadran, V., Windolf, C., and Serre, T.
\newblock Learning {L}ong-range {S}patial {D}ependencies with {H}orizontal
  {G}ated {R}ecurrent {U}nits.
\newblock In Bengio, S., Wallach, H., Larochelle, H., Grauman, K.,
  Cesa-Bianchi, N., and Garnett, R. (eds.), \emph{Advances in Neural
  Information Processing Systems 31}, pp.\  152--164. Curran Associates, Inc.,
  2018.

\bibitem[Luo et~al.(2016)Luo, Li, Urtasun, and Zemel]{luo2016understanding}
Luo, W., Li, Y., Urtasun, R., and Zemel, R.
\newblock Understanding the effective receptive field in deep convolutional
  neural networks.
\newblock In \emph{Proceedings of the 30th International Conference on Neural
  Information Processing Systems}, pp.\  4905--4913, 2016.

\bibitem[Nayebi et~al.(2018)Nayebi, Bear, Kubilius, Kar, Ganguli, Sussillo,
  DiCarlo, and Yamins]{NIPS2018_7775}
Nayebi, A., Bear, D., Kubilius, J., Kar, K., Ganguli, S., Sussillo, D.,
  DiCarlo, J., and Yamins, D.
\newblock Task-driven {C}onvolutional {R}ecurrent {M}odels of the {V}isual
  {S}ystem.
\newblock In Bengio, S., Wallach, H., Larochelle, H., Grauman, K.,
  Cesa-Bianchi, N., and Garnett, R. (eds.), \emph{Advances in Neural
  Information Processing Systems 31}, pp.\  5290--5301. Curran Associates,
  Inc., 2018.

\bibitem[Petkov \& Visser(2005)Petkov and Visser]{Petkov2005ModificationsOC}
Petkov, N. and Visser, W.
\newblock Modifications of {C}enter-{S}urround, {S}pot {D}etection and
  {D}ot-{P}attern {S}elective {O}perators.
\newblock Technical Report 2005-9-01, Institute of Mathematics and Computing
  Science, University of Groningen, Netherlands, 2005.

\bibitem[Rodieck(1965)]{rodieck_1965}
Rodieck, R.
\newblock Quantitative {A}nalysis of {C}at {R}etinal {G}anglion {C}ell
  {R}esponse to {V}isual {S}timuli.
\newblock \emph{Vision Research}, 5\penalty0 (12):\penalty0 583–601, 1965.

\bibitem[Sarma et~al.(2018)Sarma, Lee, Portegys, Ghayoomie, Jacobs, Alicea,
  Cantarelli, Currie, Gerkin, Gingell, et~al.]{sarma2018openworm}
Sarma, G.~P., Lee, C.~W., Portegys, T., Ghayoomie, V., Jacobs, T., Alicea, B.,
  Cantarelli, M., Currie, M., Gerkin, R.~C., Gingell, S., et~al.
\newblock Openworm: overview and recent advances in integrative biological
  simulation of caenorhabditis elegans.
\newblock \emph{Philosophical Transactions of the Royal Society B},
  373\penalty0 (1758):\penalty0 20170382, 2018.

\bibitem[Shapley \& Perry(1986)Shapley and Perry]{SHAPLEY1986229}
Shapley, R. and Perry, V.
\newblock Cat and monkey retinal ganglion cells and their visual functional
  roles.
\newblock \emph{Trends in Neurosciences}, 9:\penalty0 229 -- 235, 1986.

\bibitem[Srivastava et~al.(2014)Srivastava, Hinton, Krizhevsky, Sutskever, and
  Salakhutdinov]{JMLR:v15:srivastava14a}
Srivastava, N., Hinton, G., Krizhevsky, A., Sutskever, I., and Salakhutdinov,
  R.
\newblock Dropout: {A} {S}imple {W}ay to {P}revent {N}eural {N}etworks from
  {O}verfitting.
\newblock \emph{Journal of Machine Learning Research}, 15\penalty0
  (56):\penalty0 1929--1958, 2014.

\bibitem[Strisciuglio et~al.(2019)Strisciuglio, Azzopardi, and
  Petkov]{Strisciuglio2019}
Strisciuglio, N., Azzopardi, G., and Petkov, N.
\newblock Robust inhibition-augmented operator for delineation of curvilinear
  structures.
\newblock \emph{Ieee transactions on image processing}, 28\penalty0
  (12):\penalty0 5852--5866, December 2019.
\newblock ISSN 1057-7149.
\newblock \doi{10.1109/TIP.2019.2922096}.

\bibitem[Wang et~al.(2019)Wang, Yang, Xie, and Yuan]{wang2019kervolutional}
Wang, C., Yang, J., Xie, L., and Yuan, J.
\newblock Kervolutional neural networks.
\newblock In \emph{Proceedings of the IEEE/CVF Conference on Computer Vision
  and Pattern Recognition}, pp.\  31--40, 2019.

\bibitem[Zaghloul et~al.(2003)Zaghloul, Boahen, and Demb]{Zaghloul2645}
Zaghloul, K., Boahen, K., and Demb, J.
\newblock Different circuits for on and off retinal ganglion cells cause
  different contrast sensitivities.
\newblock \emph{Journal of Neuroscience}, 23\penalty0 (7):\penalty0 2645--2654,
  2003.

\bibitem[Zoumpourlis et~al.(2017)Zoumpourlis, Doumanoglou, Vretos, and
  Daras]{zoumpourlis2017non}
Zoumpourlis, G., Doumanoglou, A., Vretos, N., and Daras, P.
\newblock Non-linear convolution filters for cnn-based learning.
\newblock In \emph{Proceedings of the IEEE International Conference on Computer
  Vision}, pp.\  4761--4769, 2017.

\end{thebibliography}
\bibliographystyle{icml2021}

\appendix
\onecolumn
\beginsupplement

\section{Theoretical Proofs and Calculations}

In this section, we bring the mathematical calculations and theoretical proofs.
\subsection{Proof of Proposition 1}
\label{sec:thm1}

The DoG model used in the main document is defined as in Equation~\eqref{eq:petrov}, where $\gamma$ with $\gamma\,{<}\,1$, defines the ratio between the radius $r$ of the center and that of the surround. This model allows us to analytically compute the variances, from the size of the receptive fields: 

\begin{equation}
\label{eq:petrov-supp}
DoG_{\sigma,\gamma}(x,y) =  \frac{A_c}{\gamma^2}\, e^{-\frac{x^2+y^2}{2\gamma^2\sigma^2}} - A_s\,e^{-\frac{x^2+y^2}{2\sigma^2}}
\end{equation}

The coefficients $A_c$ and $A_s$ are determined, by requiring that the sum of all positive values in Equation~\eqref{eq:petrov-supp} are equal to those of the negative values. Here, we make them to sum up to 1 and to -1, respectively:

\begin{equation}
\iint [DoG_{\sigma,\gamma}(x,y)]^{+} dx dy = 1,
\end{equation}
\begin{equation}
\iint [DoG_{\sigma,\gamma}(x,y)]^{-} dx dy = -1
\end{equation}
By $[z]^{+}$ and $[z]^{-}$ we denote the positive and the negative half wave rectification functions, respectively:
\begin{equation}
[z]^{+} = max(0,z),
\quad [z]^{-} = min(0,z)
\end{equation}

\begin{proposition}[DoG Coefficients]
\label{theo:infinity-1} In the infinite continuous case, the coefficients $A_c$ and $A_s$ are equal. 
\end{proposition}
\begin{proof}
We have the following equalities:

\begin{equation}
\label{eq:centersum}
\iint_{\mathbb{R}^2} [\frac{A_c}{\gamma^2}\, e^{-\frac{x^2+y^2}{2\gamma^2\sigma^2}} - A_s\,e^{-\frac{x^2+y^2}{2\sigma^2}}]^{+} dx dy = 1,
\end{equation}
\begin{equation}
\label{eq:surroundsum}
\iint_{\mathbb{R}^2} [\frac{A_c}{\gamma^2}\, e^{-\frac{x^2+y^2}{2\gamma^2\sigma^2}} - A_s\,e^{-\frac{x^2+y^2}{2\sigma^2}}]^{-} dx dy = -1
\end{equation}

By transforming the integrals to the polar coordinates we have the equations below, where $r_s$ is the radius of the surround, and $r_s\to\infty$.

\begin{equation}
\int_{0}^{2\pi}\int_{0}^{r_s} [\frac{A_c}{\gamma^2}\,r\, e^{-\frac{r^2}{2\gamma^2\sigma^2}} - A_{s}\,r\,e^{-\frac{r^2}{2\sigma^2}}]^{+} dr d\theta = 1,
\end{equation}
\begin{equation}
\int_{0}^{2\pi}\int_{0}^{r_s} [\frac{A_c}{\gamma^2}\,r\, e^{-\frac{r^2}{2\gamma^2\sigma^2}} - A_{s}\,r\,e^{-\frac{r^2}{2\sigma^2}}]^{-} dr d\theta = -1
\end{equation}

The positive values are in the center with radius of $r_c$ and the negative values are in a ring between the center and surround. So we can remove the half wave rectifiers as follows: 

\begin{equation}
\label{eq:centint}
2\pi\int_{0}^{r_c} \frac{A_c}{\gamma^2}\,r\, e^{-\frac{r^2}{2\gamma^2\sigma^2}} - A_{s}\,r\,e^{-\frac{r^2}{2\sigma^2}} \,dr = 1,
\end{equation}
\begin{equation}
\label{eq:surrtint}
2\pi\int_{r_c}^{r_s} \frac{A_c}{\gamma^2}\,r\, e^{-\frac{r^2}{2\gamma^2\sigma^2}} - A_{s}\,r\,e^{-\frac{r^2}{2\sigma^2}} \,dr = -1
\end{equation}

After calculating the integrals:

\begin{equation}
2\pi(A_c\,\sigma^2\, e^{-\frac{r^2}{2\gamma^2\sigma^2}} - A_{s}\,\sigma^2\,e^{-\frac{r^2}{2\sigma^2}})\Big|_0^{r_c} = 1, \nonumber
\end{equation}
\begin{equation}
2\pi(A_c\,\sigma^2\, e^{-\frac{r^2}{2\gamma^2\sigma^2}} - A_{s}\,\sigma^2\,e^{-\frac{r^2}{2\sigma^2}})\Big|_{r_c}^{r_s} = -1 \nonumber
\end{equation}
\begin{equation}
2\pi\sigma^2(A_c\, e^{-\frac{r_c^2}{2\gamma^2\sigma^2}} - A_{s}\,e^{-\frac{r_c^2}{2\sigma^2}}) - 2\pi\sigma^2(A_c\, e^{0} - A_{s}\,e^{0}) = 1, \nonumber
\end{equation}
\begin{equation}
2\pi\sigma^2(A_c\, e^{-\frac{r_s^2}{2\gamma^2\sigma^2}} - A_{s}\,e^{-\frac{r_s^2}{2\sigma^2}}) - 2\pi\sigma^2(A_c\, e^{-\frac{r_c^2}{2\gamma^2\sigma^2}} - A_{s}\,e^{-\frac{r_c^2}{2\sigma^2}}) = -1
\end{equation}

Adding the two equations together, we have:

\begin{equation}
\lim_{r_s \to\infty}2\pi\sigma^2(A_c\, e^{-\frac{r_s^2}{2\gamma^2\sigma^2}} - A_{s}\,e^{-\frac{r_s^2}{2\sigma^2}}) - 2\pi\sigma^2(A_c - A_{s}) = 0 \nonumber
\end{equation} 
\begin{equation}
= 2\pi\sigma^2(A_c\, e^{-\infty} - A_{s}\,e^{-\infty}) - 2\pi\sigma^2(A_c - A_{s}) = 0 \nonumber
\end{equation}
\begin{equation}
2\pi\sigma^2(A_c - A_{s}) = 0 \Rightarrow A_c = A_s
\end{equation}
\end{proof}

\subsection{Computation of the Variance}
\label{subsec:oo}
The $DoG_{\sigma,\gamma}(x,y)$ is equal to zero on the border of the center and surround. The radius equals to $r_c$ on this border, meaning that $x^2+y^2=r_c^2$ when $DoG_{\sigma,\gamma}(x,y) = 0$. so by setting the $DoG_{\sigma,\gamma}(x,y) = 0$ we have:

\begin{equation}
\frac{A_c}{\gamma^2}\, e^{-\frac{r_c^2}{2\gamma^2\sigma^2}} - A_s\,e^{-\frac{r_c^2}{2\sigma^2}} = 0 \nonumber
\end{equation}

\begin{equation}
\ln(A_c) - 2\ln(\gamma) - \frac{r_c^2}{2\gamma^2\sigma^2} - \ln(A_s) + \frac{r_c^2}{2\sigma^2} = 0 \nonumber
\end{equation}
\begin{equation}
\ln(A_c)- \ln(A_s) - 2\ln(\gamma) = \frac{r_c^2}{2\gamma^2\sigma^2}  - \frac{r_c^2}{2\sigma^2} \nonumber
\end{equation}
\begin{equation}
\ln(\frac{A_c}{A_s}) - 2\ln(\gamma) = \frac{r_c^2(1-\gamma^2)}{2\gamma^2\sigma^2} \nonumber
\end{equation}
\begin{equation}
\sigma^2 = \frac{r_c^2(1-\gamma^2)}{2\gamma^2(\ln(\frac{A_c}{A_s}) - 2\ln(\gamma))} \nonumber
\end{equation}
\begin{equation}
\sigma = \frac{r_c}{\gamma}\sqrt{\frac{1-\gamma^2}{2\ln(\frac{A_c}{A_s})-4\ln(\gamma)}} 
\end{equation}

Based on Proposition 1, the values of $A_c$ and $A_s$ are equal in the infinite continuous case. Since in the finite discrete case those values are very close, we can approximate the value of $\sigma$: 

\begin{equation}
\label{eq:sigma-1}
\sigma \approx \frac{r_c}{2\gamma}\sqrt{\frac{1-\gamma^2}{-\ln{\gamma}}}
\end{equation}

\subsection{Proof of Theorem 1}
\label{subsec:thm2}

We use Equation~\eqref{eq:petrov-supp} to compute the weights in the On-center kernel matrix $DoG_{\rm On}$. For the Off-center kernel $DoG_{\rm Off}$, we use the same equation with the signs inverted. For a given input  $\chi$, we calculate the On and Off responses by convolving $\chi$ with the computed fixed kernels separately:
\begin{equation}
\chi_{\rm On}[x,y] = (\chi * DoG_{\rm On})[x,y],
\end{equation}
\begin{equation}
\chi_{\rm Off}[x,y] = (\chi * DoG_{\rm Off})[x,y]
\end{equation}

Note that the On and Off convolutions cover the input image completely. These two convolutions result in the following equations, when the kernel is in the shape of a square:

\begin{equation}
\label{eqn:on}
\chi_{\rm On}[x,y] 
= 
\int_{-r_s}^{r_s}\int_{-r_s}^{r_s} \chi(x+\rho,y+\tau)\,(\frac{A_c}{\gamma^2}\, e^{-\frac{\rho^2+\tau^2}{2\gamma^2\sigma^2}} - A_s\,e^{-\frac{\rho^2+\tau^2}{2\sigma^2}})\,d\rho\,d\tau
\end{equation}
\begin{equation}
\label{eqn:off}
\chi_{\rm Off}[x,y] 
= 
\int_{-r_s}^{r_s}\int_{-r_s}^{r_s} \chi(x+\rho,y+\tau)\,(A_s\, e^{-\frac{\rho^2+\tau^2}{2\sigma^2}} - \frac{A_c}{\gamma^2}\,e^{-\frac{\rho^2+\tau^2}{2\gamma^2\sigma^2}})\, d\rho\,d\tau
\end{equation}

\textbf{Proof of Theorem \ref{theo:main}
The on- and off-path\-ways learn unique and complementary features.
%The features extracted by the On convolution are not identical to the features extracted by the Off convolution and vice versa.
}

\begin{proof}
We prove this theorem by contradiction. Assume that the features extracted by the On convolution are identical to the features extracted by the Off convolution. Now suppose the input image has a small spot of light (smaller than the center of our kernels) on a dark background. We first convolve this image with an On kernel. If the spot of light lies in the center of the kernel, the convolution will result in a response close to $1$, according to the Equations~\eqref{eq:centersum},~\eqref{eq:surroundsum}, and ~\eqref{eqn:on}. If the spot of light lies in the surround, the convolution will result in a  negative response. As a result we obtain an activation map with values close to $1$ where the light spot is located, negative values in the outer edges of the light spot, and zero values everywhere else. 

When we convolve the same image with an Off kernel we obtain the following. If the spot of light lies in the center of the kernel, then we obtain a negative response. If it lies in the surround, then it will result in a positive response close to zero, according to the Equations~\eqref{eq:centersum},~\eqref{eq:surroundsum}, and~\eqref{eqn:off}. Hence, the Off convolution results in an activation map with small values in the outer edges of the light spot, negative values where the light spot is located, and zero values everywhere else. Comparing the two activation maps, one can see that: 1)~The positive values are in different locations, and 2)~These values are close to $1$ for the On activation map, and close to zero for the Off activation map. This contradicts our initial assumption. Note that a similar but complementary argument can be made for an image with a small dark spot on a light background.
\end{proof}

\section{Experimental Setup}
Here, we describe the experimental setup for the tasks discussed in Tables 1, 2, 3, 4 and 5.

\subsection{Dataset description}
\textbf{Image classification.} We used a random subset of the Imagenet dataset~\cite{DBLP:conf/cvpr/DengDSLL009} with 60000 images from 100 categories. We used 5000 of the images for each of the validation and test sets. All samples were cropped around the center if they were not originally in square shapes, and then resized to $192\times192$ pixels.

\textbf{Robustness to illumination change.} We used small Norb dataset~\cite{DBLP:conf/cvpr/LeCunHB04} which contains images of toys from 5 generic categories: human figures, four-legged animals, airplanes, cars and trucks. The images of each category were taken from 10 toy instances in 6 different lighting conditions, 9 elevations, and 18 azimuths. The training set consists of the images from 5 of the instances of each category, and the rest 5 instances are in the test set. Figure~\ref{fig:norbs} shows one sample from each category in each of the lighting conditions. We separated the dataset based on the lighting conditions, and used the images from the first light, Light0, as the training set and tested our networks on testsets from all 6 different lighting conditions. Each of the training and test sets contained 4050 images of size $96\times96$ pixels.

\textbf{Robustness to distribution shifts.} We used MNIST dataset~\cite{lecun-mnisthandwrittendigit-2010} containing grayscaled images of handwritten digits. There are 60000 and 10000 samples in the training and test sets respectively. For testing, we inverted all the pixel values by subtracting them form 255.

\begin{figure}[t]
    \centering
    \includegraphics[width=0.8\textwidth]{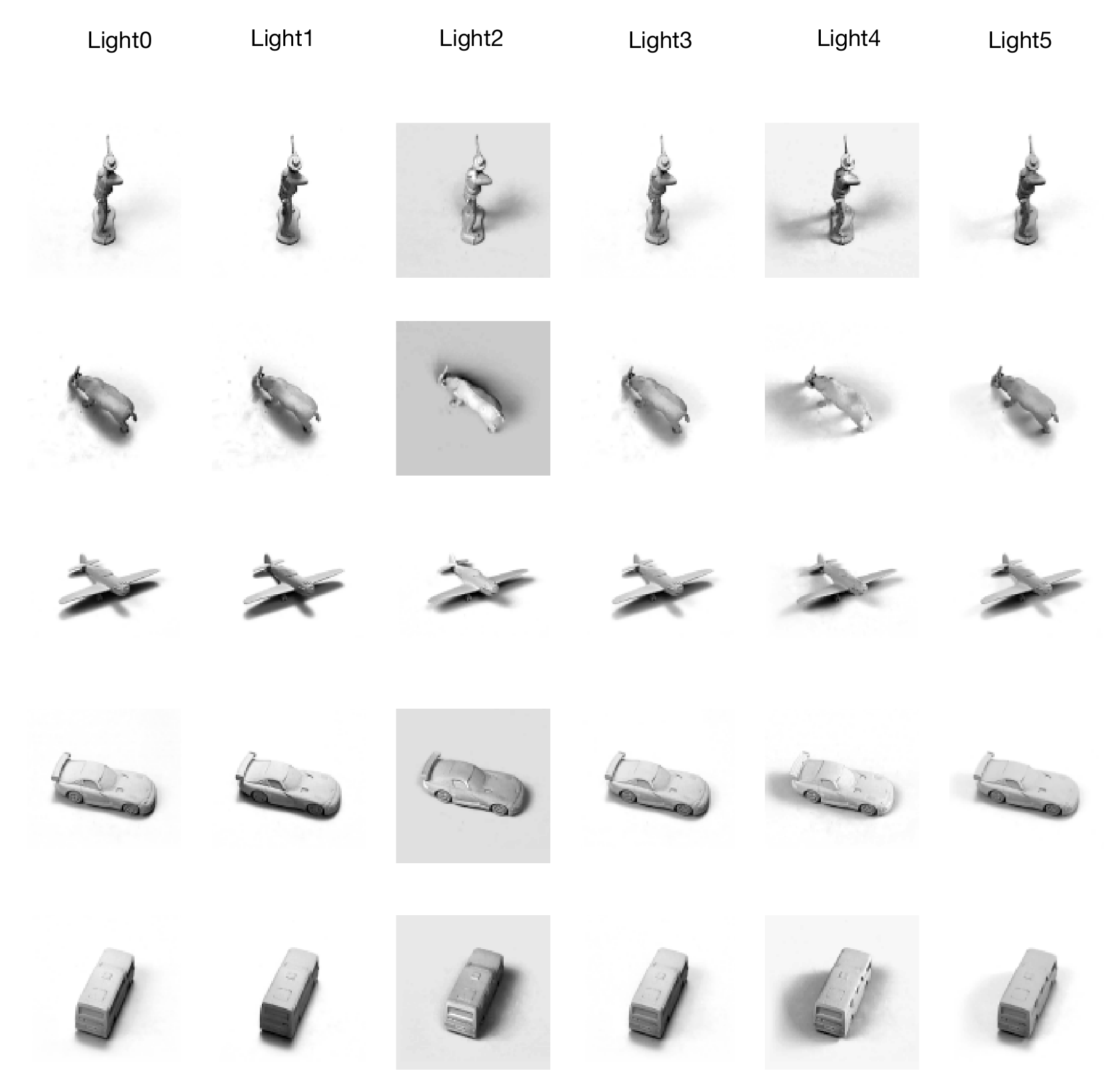}
    \caption{Example samples of the Norb dataset~\cite{DBLP:conf/cvpr/LeCunHB04} from each of the five categories and each of the lighting conditions.}
    \label{fig:norbs}
\end{figure}

\subsection{Network architectures and Hyper parameters}

For each of the experiments, we used a CNN as the base network, with different numbers of layers depending the dataset image sizes. Figure~\ref{fig:archs} shows the architectures of the base networks. We construct the other models from the base networks as discussed in the main paper. For the On and Off Center convolutions in OOCS-CNNs, we used kernels of size $5\times5$ for Imagenet and Norb datasets. We used smaller kernels of size $3\times3$ for the MNIST dataset, since the images are of smaller size. We calculated the On and Off resposes from the inputs and directly fed their summation to the network.

\begin{figure}[t]
    \centering
    \includegraphics[width=0.85\textwidth]{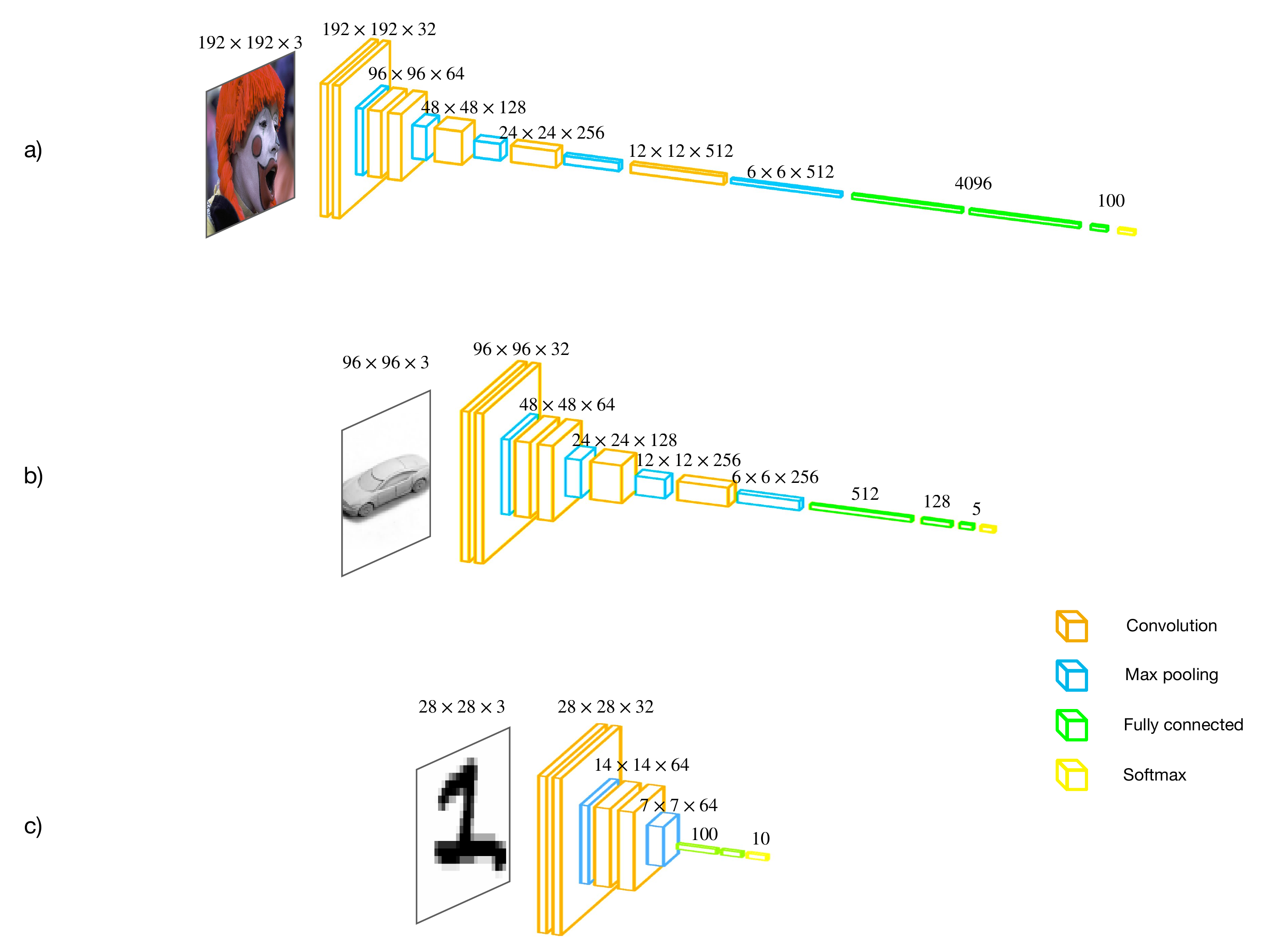}
    \caption{Base network architectures for a) Imagenet subset classification, b) robustness evaluation on Norb, and c) robustness evaluation on MNIST}
    \label{fig:archs}
\end{figure}

We had batch sizes of 64 in all experiments. We used Adam optimiser~\cite{DBLP:journals/corr/KingmaB14} for experiments on Imagenet and Norb, with a learning rate of $10^{-4}$. In the experiment on Imagenet, we decreased the learning rate to half after 10 epochs which was mainly in favour of the baselines. In the Imagenet experiments with ResNet-34 we use SGD optimiser and start with a learning rate of 0.1, which we decay by a factor of 0.1 every 20 epochs and we trained the networks for 60 epochs. For scaling the gradient descent steps, we use a Nesterov-momentum of 0.9.

\begin{figure}[h]
    \centering
    \includegraphics[width=0.7\textwidth]{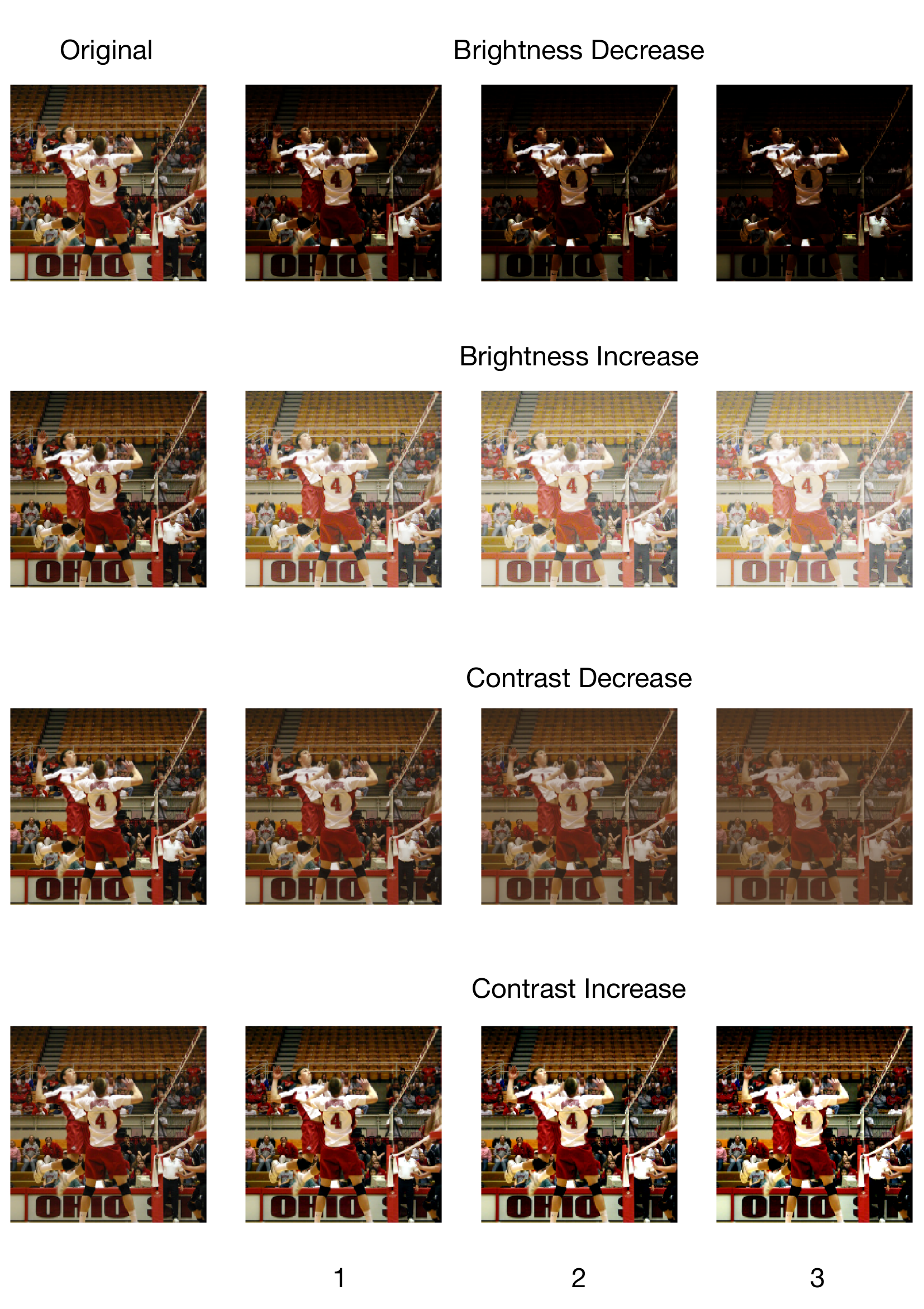}
    \caption{Sample Brightness and contrast variations we tested OOCS and Residual networks against.}
    \label{fig:robustess_sample}
\end{figure}

\begin{figure}[h]
    \centering
    \includegraphics[width=1\textwidth]{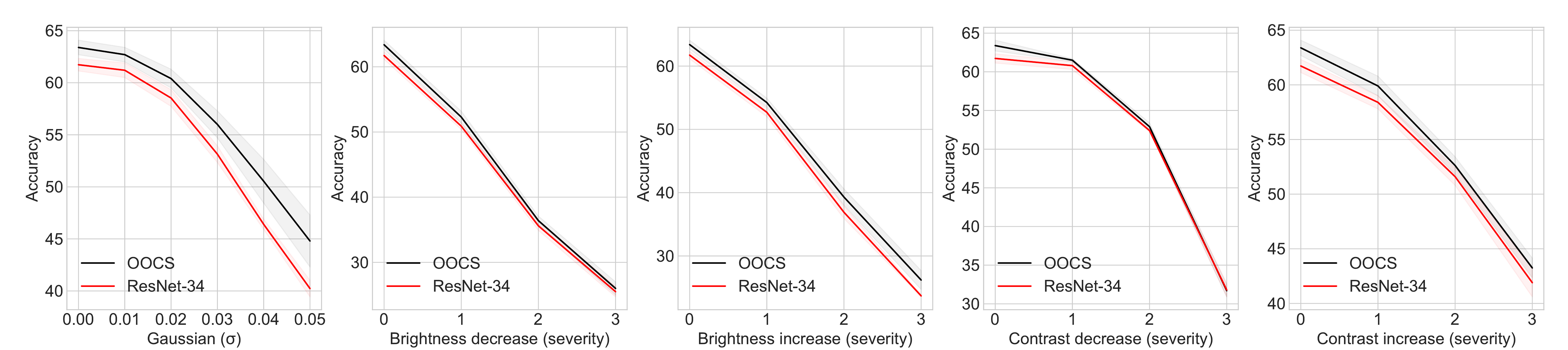}
    \caption{OOCS filters added to ResNet-34 consistently enhances the robustness of a network to perturbations such as Gaussian noise, Brightness and Contrast variations.}
    \label{fig:results_plots}
\end{figure}

\section{Experiments on Digital Distribution Shifts}

In this section we describe the experiments to evaluate the robustness of a ResNet-34 on the Imagnet subset compared to the same network equipped with OOCS.

We altered the test set images with 5 different digital perturbations: adding Gaussian noise, decreasing and increasing the brightness (Gamma correction), and decreasing and increasing the contrast. Figure~\ref{fig:robustess_sample} shows one sample of Imagenet dataset in different brightness and contrast changes with different severities.

The results of this experiment are summarized in tables S1-5 and figure~\ref{fig:results_plots}. As the results show, OOCS can enhance the robustness of a ResNet under digital distribution shifts.

\begin{table}[h]
  %\fontsize{9.0pt}{10.8pt} \selectfont
   
  \centering
  \caption{Test Accuracy and variance for test images with Gaussian noise. n=6.}
  \vspace*{1.5ex}
  \begin{tabular}{lccccc}
    \toprule
    \multicolumn{1}{c}{} & \multicolumn{5}{c}{\textbf{Gaussian Noise ($\sigma$)}} \\
    \midrule
    \textbf{Models} &0.01&0.02&0.03&0.04&0.05\\
    \midrule
    ResNet-34&$61.2_{\pm0.7}$&$58.5_{\pm0.8}$&$53.14_{\pm0.8}$&$46.4_{\pm0.5}$&$40.23_{\pm0.8}$\\
    OOCS-ResNet-34&$\mathbf{62.7_{\pm0.7}}$&$\mathbf{60.4_{\pm0.9}}$&$\mathbf{56.0_{\pm1.3}}$&$\mathbf{50.5_{\pm2.1}}$ &$\mathbf{44.8_{\pm2.5}}$\\
    \bottomrule
  \end{tabular}
  \label{light0noisegaussian-table-gaus}
\end{table}

\begin{table}[h]
  %\fontsize{9.0pt}{10.8pt} \selectfont
   
  \centering
  \caption{Test Accuracy and variance for test images with decreasing brightness. n=6.}
  \vspace*{1.5ex}
  \begin{tabular}{lccc}
    \toprule
    \multicolumn{1}{c}{} & \multicolumn{3}{c}{\textbf{Gamma Correction ($\gamma$)}} \\
    \midrule
    \textbf{Models} &2&3&4\\
    \midrule
    ResNet-34&$50.9_{\pm0.5}$&$35.6_{\pm0.7}$&$25.5_{\pm0.8}$\\
    OOCS-ResNet-34&$\mathbf{52.3_{\pm0.9}}$&$\mathbf{36.4_{\pm0.7}}$&$\mathbf{26.0_{\pm1.0}}$\\
    \bottomrule
  \end{tabular}
  \label{light0noisegaussian-table-brightness}
\end{table}

\begin{table}[h]
  %\fontsize{9.0pt}{10.8pt} \selectfont
   
  \centering
  \caption{Test Accuracy and variance for test images with increasing brightness. n=6.}
  \vspace*{1.5ex}
  \begin{tabular}{lccc}
    \toprule
    \multicolumn{1}{c}{} & \multicolumn{3}{c}{\textbf{Gamma Correction ($\gamma$)}} \\
    \midrule
    \textbf{Models} &1/2&1/3&1/4\\
    \midrule
    ResNet-34&$52.7_{\pm0.7}$&$36.9_{\pm1.0}$&$23.7_{\pm0.2}$\\
    OOCS-ResNet-34&$\mathbf{54.2_{\pm0.7}}$&$\mathbf{39.3_{\pm0.7}}$&$\mathbf{26.2_{\pm1.5}}$\\
    \bottomrule
  \end{tabular}
  \label{light0noisegaussian-table-inc-brightness}
\end{table}

\begin{table}[h]
  %\fontsize{9.0pt}{10.8pt} \selectfont
   
  \centering
  \caption{Test Accuracy and variance for test images with decreasing contrast. n=6.}
  \vspace*{1.5ex}
  \begin{tabular}{lccc}
    \toprule
    \multicolumn{1}{c}{} & \multicolumn{3}{c}{\textbf{Contrast Factor}} \\
    \midrule
    \textbf{Models} &0.8&0.6&0.4\\
    \midrule
    ResNet-34&$60.8_{\pm0.5}$&$52.4_{\pm0.2}$&$\mathbf{31.8_{\pm0.9}}$\\
    OOCS-ResNet-34&$\mathbf{61.5_{\pm0.2}}$&$\mathbf{52.9_{\pm0.5}}$&$31.7_{\pm0.8}$\\
    \bottomrule
  \end{tabular}
  \label{light0noisegaussian-table-dec-contrast}
\end{table}

\begin{table}[t]
  %\fontsize{9.0pt}{10.8pt} \selectfont
   
  \centering
  
  \caption{Test Accuracy and variance for test images with increasing contrast. n=6.}
  \vspace*{1.5ex}
  \begin{tabular}{lccc}
    \toprule
    \multicolumn{1}{c}{} & \multicolumn{3}{c}{\textbf{Contrast Factor}} \\
    \midrule
    \textbf{Models} &1.2&1.4&1.6\\
    \midrule
    ResNet-34&$58.4_{\pm0.6}$&$51.6_{\pm0.8}$&$41.9_{\pm1.3}$\\
    OOCS-ResNet-34&$\mathbf{59.9_{\pm0.9}}$&$\mathbf{52.62_{\pm0.8}}$&$\mathbf{43.25_{\pm0.9}}$\\
    \bottomrule
  \end{tabular}
  \label{light0noisegaussian-table-inc-contrast}
\end{table}

\clearpage

\section{Code and Data Availability}
All code and data are included in \url{https://github.com/ranaa-b/OOCS}.

% \appendix
% \section{Do \emph{not} have an appendix here}

% \textbf{\emph{Do not put content after the references.}}
% %
% Put anything that you might normally include after the references in a separate
% supplementary file.

% We recommend that you build supplementary material in a separate document.
% If you must create one PDF and cut it up, please be careful to use a tool that
% doesn't alter the margins, and that doesn't aggressively rewrite the PDF file.
% pdftk usually works fine. 

% \textbf{Please do not use Apple's preview to cut off supplementary material.} In
% previous years it has altered margins, and created headaches at the camera-ready
% stage. 
%%%%%%%%%%%%%%%%%%%%%%%%%%%%%%%%%%%%%%%%%%%%%%%%%%%%%%%%%%%%%%%%%%%%%%%%%%%%%%%
%%%%%%%%%%%%%%%%%%%%%%%%%%%%%%%%%%%%%%%%%%%%%%%%%%%%%%%%%%%%%%%%%%%%%%%%%%%%%%%

\end{document}